\documentclass[11pt]{article}

\usepackage[final]{acl}
\usepackage{times}
\usepackage{latexsym}
\usepackage{amsmath}
\usepackage{amssymb}
\usepackage{booktabs}
\usepackage{xcolor}
\usepackage{colortbl} 
\usepackage{tcolorbox}
\usepackage{algorithm}
\usepackage{algpseudocode} 
\usepackage{multirow}
\usepackage{amsthm}
\newtheorem{theorem}{Theorem}
\newtheorem{definition}{Definition}
\newtheorem{lemma}{Lemma}
\newtheorem{remark}{Remark}
\usepackage{enumitem}
\usepackage[T1]{fontenc}
\usepackage[utf8]{inputenc}
\usepackage{microtype}
\usepackage{inconsolata}
\usepackage{graphicx}
\usepackage{subfigure}

\title{How to Allocate, How to Learn? Dynamic Rollout Allocation and Advantage Modulation for Policy Optimization}

\author{
    Yangyi Fang\textsuperscript{1,2,*}, 
    Jiaye Lin\textsuperscript{1,*}, 
    Xiaoliang Fu\textsuperscript{1,3,*}, 
    Cong Qin\textsuperscript{1,4}, 
    Haolin Shi\textsuperscript{2}, \\
    \textbf{Chaowen Hu}\textsuperscript{1}, 
    \textbf{Lu Pan}\textsuperscript{1}, 
    \textbf{Ke Zeng}\textsuperscript{1}, 
    \textbf{Xunliang Cai}\textsuperscript{1,$\dagger$} \\
    \textsuperscript{1}Meituan \quad
    \textsuperscript{2}Tsinghua University \quad
    \textsuperscript{3}Fudan University \quad
    \textsuperscript{4}Peking University \\
    \texttt{\{fangyangyi, linjiaye\}@meituan.com}
}

\begin{document}

\maketitle

\begingroup
  \renewcommand\thefootnote{}
  \footnotetext{
    \textsuperscript{*}\ Equal contribution. \quad 
    \textsuperscript{$\dagger$}\ Corresponding author.
  }
\endgroup

\begin{abstract}
Reinforcement Learning with Verifiable Rewards (RLVR) has proven effective for Large Language Model (LLM) reasoning, yet current methods face key challenges in resource allocation and policy optimization dynamics: (i) uniform rollout allocation ignores gradient variance heterogeneity across problems, and (ii) the softmax policy structure causes gradient attenuation for high-confidence correct actions, while excessive gradient updates may destabilize training. Therefore, we propose \textbf{DynaMO}, a theoretically-grounded dual-pronged optimization framework. \textit{At the sequence level}, we prove that uniform allocation is suboptimal and derive variance-minimizing allocation from the first principle, establishing Bernoulli variance as a computable proxy for gradient informativeness. \textit{At the token level}, we develop gradient-aware advantage modulation grounded in theoretical analysis of gradient magnitude bounds. Our framework compensates for gradient attenuation of high-confidence correct actions while utilizing entropy changes as computable indicators to stabilize excessive update magnitudes. Extensive experiments conducted on a diverse range of mathematical reasoning benchmarks demonstrate consistent improvements over strong RLVR baselines. Our implementation is available at: \href{https://github.com/FlyTune/DynaMO-RL}{GithubX-F/DynaMO-RL}.
\end{abstract}

\section{Introduction}
\label{sec:intro}
Reinforcement Learning with Verifiable Rewards (RLVR) has recently emerged as a powerful and promising paradigm for advancing Large Language Model (LLM) reasoning~\cite{ouyang2022training, bai2022training}. Recent breakthrough models, such as OpenAI o1~\cite{jaech2024openai} and DeepSeek-R1~\cite{deepseekai2025deepseekr1incentivizingreasoningcapability}, demonstrate emergent capabilities like long-form chain-of-thought and self-reflection. Building on these successes, numerous works have explored RLVR methods~\cite{shao2024deepseekmath, fang2026placingpuzzlepiecesmatter}, most commonly combining the GRPO algorithm~\cite{shao2024deepseekmath} or its variants~\cite{liu2025understanding, yu2025dapo} with outcome-based rewards for reinforcement learning. However, despite these advances, fundamental challenges persist in both computational resource allocation and policy optimization dynamics.

Current RLVR methods uniformly distribute rollout budgets across training instances~\cite{shao2024deepseekmath}, ignoring heterogeneous gradient informativeness among various problems. This overlooks a fundamental trade-off: while high-variance problems contribute more informative learning signals, they simultaneously introduce greater estimation noise that may destabilizes model performance. Existing adaptive strategies either prioritize sample selection over resource allocation~\cite{bengio2009curriculum, schaul2015prioritized, tong2024dart} or target different optimization frameworks~\cite{dong2023raft, yao2025optimizing}, leaving the variance-informativeness trade-off in policy gradient methods unaddressed. Moreover, this limitation is particularly pronounced on datasets with diverse difficulties, where effective allocation requires dynamically balancing informativeness and noise.

Compounding this resource allocation challenge, the policy gradient dynamics in RLVR training expose fundamental token-level optimization issues. Theoretical analysis shows that the mathematical structure of the softmax policy induces an inherent and problematic tension~\cite{li2025logit}: high-confidence correct actions naturally yield smaller gradient magnitudes, leading to insufficient learning signals, whereas excessive gradient updates may severely destabilize training~\cite{luo2025deepscaler}. Our analysis further reveals that gradient magnitude is provably upper-bounded by policy entropy, allowing entropy changes to serve as computable indicators of such instability. Existing approaches attempt to mitigate this issue via ratio clipping~\cite{yu2025dapo, yang2025dcpo}, sample reweighting~\cite{zhu2025surprising}, or entropy-induced advantages~\cite{cheng2025reasoning, tan2025gtpo, wang2025beyond}, but these coarse-grained interventions lack a unified theoretical grounding and even amplify rather than stabilize training fluctuations.

To address these key challenges, we propose \textbf{DynaMO}, a dual-pronged optimization framework grounded in variance minimization theory and policy gradient analysis. \textit{At the sequence level}, we derive dynamic rollout allocation that explicitly balances the informativeness-noise trade-off by gradient variance minimization specifically for policy gradient methods, establishing Bernoulli variance as a lightweight computable proxy. \textit{At the token level}, we introduce gradient-aware advantage modulation through integrated compensation and stabilization mechanisms based on our gradient-entropy analysis: compensation for gradient attenuation in high-confidence correct actions and stabilization against excessive update magnitudes by monitoring entropy changes as an indicator. Our core contributions in this paper are summarized as follows:
\begin{itemize}[leftmargin=0.5cm]
\item We prove that uniform allocation is suboptimal and subsequently derive variance-minimizing rollout allocation with a lightweight proxy.
\item We establish the gradient-entropy relationship through theoretical analysis, enabling gradient-aware advantage modulation with compensation and stabilization mechanisms.
\item Extensive experiments across six benchmarks and three LLM scales demonstrate consistent improvements, with comprehensive ablations validating each component and visualizations revealing stable optimization dynamics.
\end{itemize}

\section{Related Works}
\subsection{Reinforcement Learning for LLMs}
Reinforcement learning has emerged as a dominant paradigm for LLM post-training, with RLHF and RLVR demonstrating significant success~\cite{ouyang2022training, bai2022training, schulman2017proximal}. Recent breakthrough models, including DeepSeek-R1~\cite{deepseekai2025deepseekr1incentivizingreasoningcapability}, DeepSeekMath~\cite{shao2024deepseekmath}, OpenAI o1~\cite{jaech2024openai}, and Kimi k1.5~\cite{team2025kimi}, further demonstrate the remarkable effectiveness of RLVR on complex reasoning tasks with verifiable rewards. While subsequent works have introduced various algorithmic refinements~\cite{liu2025understanding, yu2025dapo, chu2025gpg, hu2025open, fang2026ckpro}, fundamental challenges in computational efficiency and optimization stability still persist.

\subsection{Entropy Dynamics in Policy Optimization}
Entropy regularization balances exploration and exploitation~\cite{haarnoja2018soft, mnih2016asynchronous}, yet its precise role in LLM training remains contentious~\cite{ouyang2022training, shao2024deepseekmath, yu2025dapo, chu2025gpg}. Entropy collapse~\cite{luo2025deepscaler} motivates various mitigation strategies such as ratio clipping~\cite{yu2025dapo, yang2025dcpo}, sample reweighting~\cite{zhu2025surprising}, or entropy-induced advantages~\cite{cheng2025reasoning, tan2025gtpo, wang2025beyond, li2026cso}. \citet{li2025logit} further shows that high-confidence actions yield attenuated gradient magnitudes, thereby inducing asymmetric learning dynamics. Yet existing methods lack unified theoretical grounding, and the complex interplay between gradient attenuation and entropy dynamics remains largely underexplored. Our work addresses this through gradient-aware policy update control grounded in the gradient-entropy relationship, where entropy serves as a computable and interpretable indicator of update magnitude rather than a direct optimization target.

\subsection{Sample Efficiency}
Sample efficiency is critically important for RLVR training, where generating multiple rollouts per problem incurs substantial computational cost. Standard methods employ uniform rollout budgets~\cite{shao2024deepseekmath}, largely overlooking heterogeneous gradient informativeness across problems. Curriculum learning~\cite{bengio2009curriculum} and prioritized experience replay~\cite{schaul2015prioritized} choose which problems to train on, but primarily emphasize sample ordering rather than resource allocation. Offline methods~\cite{tong2024dart} repeatedly sample until obtaining a fixed number of correct responses, lacking dynamic scheduling for online training. EM-based methods~\cite{dong2023raft, gulcehre2023reinforced, yao2025optimizing, liu2021quadrupletbert, liu2021improving, liu2022label, yu2026rfew} enhance efficiency via iterative rejection sampling, yet require gradient norm computations and target the EM framework rather than policy gradient methods. In contrast, we derive optimal rollout allocation by minimizing gradient variance for policy gradient methods, with a lightweight, gradient-free proxy based on historical success statistics.

\begin{figure*}[t]
\centering
\includegraphics[width=1\textwidth]{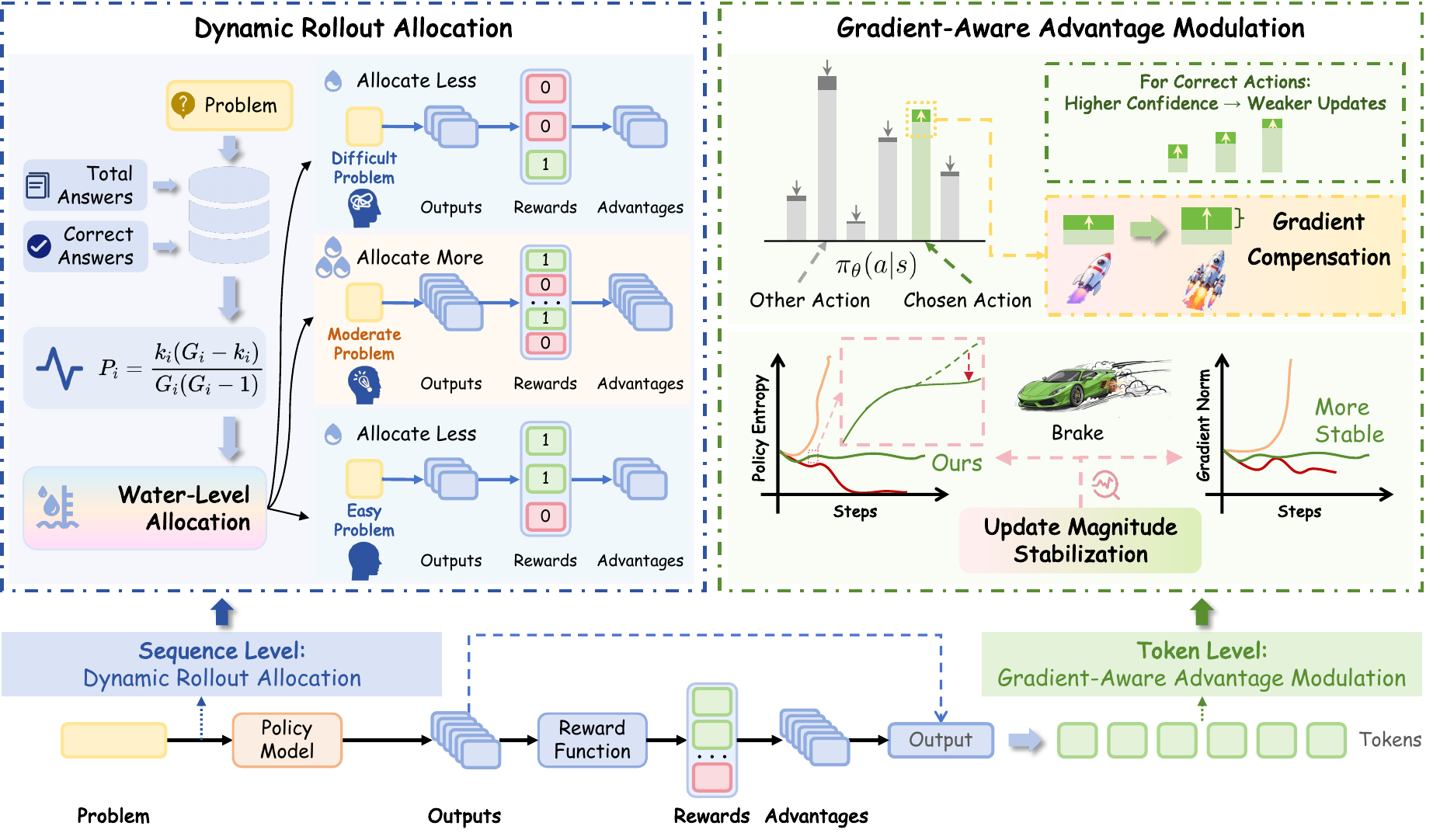}
\caption{Overview of DynaMO, which operates at both the sequence and token levels, enabling fine-grained control over optimization. (i) Left: Dynamic allocation concentrates the rollout budget on high-variance problems. (ii) Right: Gradient-aware advantage modulation compensates for attenuated gradients and stabilizes excessive updates.}
\vspace{-0.5em}
\label{fig:overview}
\end{figure*}

\section{Preliminaries}

\paragraph{Policy Optimization in RLVR.}
Formally, given a prompt $q$ sampled from the training data $\mathcal{D}$, let $\pi_{\theta}$ denotes the policy model parameterized by $\theta$. In the context of RLVR, the model autoregressively generates a response $o = \{o_1, o_2, \ldots, o_T\}$, where each token $o_t$ represents an action taken at step $t$ and $T$ represents the sequence length. The final objective is to maximize the expected reward:
\begin{equation}
\small
J(\theta) = \mathbb{E}_{q \sim \mathcal{D}, o \sim \pi_\theta(\cdot \mid q)} [R(o)],
\end{equation}
where $R(\cdot)$ is the reward function that evaluates the quality of the generated response $o$.

To reduce variance, GRPO~\cite{shao2024deepseekmath} samples $G$ responses $\{o_i\}_{i=1}^{G}$ per prompt, estimating advantages via group-wise normalization:
\begin{equation}
\small
A_{i,t} = \frac{R(o_i) - \text{mean}(\{R(o_j)\}_{j=1}^G)}{\text{std}(\{R(o_j)\}_{j=1}^G)},
\end{equation}
where $R(o_i)$ denotes the reward for response $o_i$, evaluated by the reward function $R(\cdot)$. $\text{mean}(\cdot)$ and $\text{std}(\cdot)$ denote the mean and standard deviation of the rewards within the group.

With the advantage $A_{i,t}$ shared across all tokens, GRPO maximizes the clipped surrogate objective:
\begin{equation}
\small
\begin{aligned}
\mathcal{L}(\theta) &= \mathbb{E}_{q \sim \mathcal{D}, \{o_i\}_{i=1}^{G} \sim \pi_{\text{old}}(\cdot \mid q)} \Bigg[
\frac{1}{\sum_{i=1}^{G} |o_i|} \sum_{i=1}^{G} \sum_{t=1}^{|o_i|} \\
&\quad \min \Big( r_{i,t} A_{i,t}, \text{clip}(r_{i,t}, 1-\epsilon, 1+\epsilon) A_{i,t} \Big)
\Bigg],
\end{aligned}
\end{equation}
where $r_{i,t} = \frac{\pi_\theta(o_{i,t} \mid q,o_{i,<t})}{\pi_{\text{old}}(o_{i,t} \mid q,o_{i,<t})}$ is the importance sampling ratio and $\epsilon$ is the clipping parameter. Consistent with prior works~\cite{yu2025dapo, liu2025understanding}, we omit the KL divergence penalty term and do not elaborate further on this component.

\paragraph{Policy Entropy.}
Policy entropy quantifies the uncertainty in the model's action selection process, serving as a key indicator of exploration in reinforcement learning. Given a policy model $\pi_\theta$ and training data $\mathcal{D}$, the policy entropy is defined as:
\begin{equation}
\small
\begin{aligned}
\mathcal{H}(\pi_\theta, \mathcal{D}) &= \mathbb{E}_{q \sim \mathcal{D}, \{o_i\}_{i=1}^{G} \sim \pi_\theta(\cdot \mid q)} \left[ \frac{1}{\sum_{i=1}^{G} |o_i|} \sum_{i=1}^{G} \sum_{t=1}^{|o_i|} \right. \\
&\quad \left. \left( -\sum_{v \in \mathcal{V}} \pi_\theta(v|q,o_{i,<t}) \log \pi_\theta(v|q,o_{i,<t}) \right) \right],
\end{aligned}
\end{equation}
where $\mathcal{V}$ represents the vocabulary. This entropy metric reflects the model's confidence distribution: higher values indicate greater uncertainty and exploration potential, whereas lower values suggest more deterministic behavior selection.

\section{Methodology}
\subsection{Overview}
To effectively address the challenges of inefficient resource allocation and policy optimization dynamics highlighted in Section~\ref{sec:intro}, we propose \textbf{DynaMO} (\textbf{Dyna}mic Rollout Allocation and Advantage \textbf{M}odulation for Policy \textbf{O}ptimization). As illustrated in Figure~\ref{fig:overview}, this dual-pronged framework operates at both the sequence and token levels, enabling fine-grained control over optimization.

Notably, DynaMO comprises two key components: (i) dynamic rollout allocation that adaptively distributes computational budget based on gradient variance minimization, concentrating resources on problems with balanced success-failure distributions where learning signals are most informative, and (ii) gradient-aware advantage modulation based on the gradient-entropy upper bound relationship, which compensates for gradient attenuation in high-confidence actions while using entropy changes to stabilize excessive updates.

\subsection{Dynamic Rollout Allocation}
\label{sec:dynamic_allocation}
We establish a theoretical framework for optimal rollout allocation by formulating it as a gradient variance minimization problem, which considers a training dataset $\mathcal{D}$ with $N$ prompts $\{q_i\}_{i=1}^N$.

\subsubsection{Optimization Theory}
For a prompt $q_i$ with $n_i$ rollout budget, the corresponding gradient estimator $\hat{g}_i = \frac{1}{n_i}\sum_{k=1}^{n_i} g_{i,k}$ has variance $\text{Var}[\hat{g}_i] = \sigma_i^2 / n_i$. Minimizing the total estimation variance $\sum_i \text{Var}[\hat{g}_i]$ under budget constraint $\sum_i n_i = B$ yields (proof in Appendix~\ref{sec:gradient_variance}):

\begin{equation}
\small
n_i^* = B \cdot \frac{\sigma_i}{\sum_{k=1}^N \sigma_k}.
\end{equation}

This principle allocates more rollouts to problems with higher gradient variance. To implement this, we use Bernoulli variance $P_i$ as a practical proxy for $\sigma_i$ (derivation in Appendix~\ref{sec:gradient_variance}).

\subsubsection{Bernoulli Variance as Practical Proxy}
For binary rewards, the variance follows $p(1-p)$, where $p$ represents the success probability. Then, we estimate this via historical statistics with $k_i$ correct responses out of $G_i$ total rollouts generated:
\begin{equation}
\small
P_i = \frac{k_i(G_i - k_i)}{G_i(G_i - 1)}, \quad \mathbb{E}[P_i] = p_i(1-p_i).
\end{equation}

This unbiased estimator can be efficiently computed from historical success counts, with problems that maximize $P_i$ (characterized by balanced correct/incorrect responses) exhibiting large $\sigma_i$.

\paragraph{Water-Level Implementation.}
Algorithm~\ref{alg:dynamic_allocation} implements the proportional budget allocation $n_i \propto P_i$ with the boundary constraints $G_{\min}$ (ensuring minimum coverage) and $G_{\max}$ (preventing over-concentration). After each training iteration, we incrementally update the statistics via $G_i \leftarrow G_i + G_i^{\text{new}}$ and $k_i \leftarrow k_i + k_i^{\text{new}}$, thereby enabling an adaptive allocation as model proficiency evolves. Notably, the dynamic allocation strategy provably reduces variance against the conventional uniform allocation, with theoretical justification and operational details in Appendix~\ref{sec:gradient_variance} and~\ref{sec:DRA_detail}.

\begin{algorithm}[t]
\small
\caption{Variance-Driven Dynamic Allocation}
\label{alg:dynamic_allocation}
\begin{algorithmic}[1]
\Require Historical statistics $\{(G_i, k_i)\}_{i=1}^N$, total rollout budget $B$, allocation bounds $[G_{\min}, G_{\max}]$
\State Compute priorities $P_i = k_i(G_i - k_i) / [G_i(G_i - 1)]$ 
\State Initialize $G_i^{\text{new}} = G_{\min}$, $B_{\text{rem}} = B - N \cdot G_{\min}$
\While{$B_{\text{rem}} > 0$ and $\exists i: G_i^{\text{new}} < G_{\max}$}
    \State $\mathcal{E} = \{i : G_i^{\text{new}} < G_{\max}\}$
    \For{$i \in \mathcal{E}$}
        \State $\Delta G_i = \min\left(\left\lfloor \frac{B_{\text{rem}} P_i}{\sum_{j \in \mathcal{E}} P_j} \right\rfloor, G_{\max} - G_i^{\text{new}}\right)$
        \State $G_i^{\text{new}} \leftarrow G_i^{\text{new}} + \Delta G_i$
        \State $B_{\text{rem}} \leftarrow B_{\text{rem}} - \Delta G_i$
    \EndFor
\EndWhile
\State \Return $\{G_1^{\text{new}}, \ldots, G_N^{\text{new}}\}$
\end{algorithmic}
\end{algorithm}

\subsection{Gradient-Aware Advantage Modulation}
Beyond resource allocation, the softmax policy structure causes high-confidence correct actions to produce attenuated gradient magnitudes, while excessive gradient updates may undermine training stability. Grounded in theoretical analysis of gradient dynamics, we design an integrated compensation and stabilization mechanism.

\subsubsection{Gradient Compensation}
Our gradient analysis, detailed in Appendix~\ref{sec:grad_bound}, reveals the fundamental relationship between the update magnitude and token-level entropy. Specifically, considering a softmax policy parameterized by logits $z(s)$ at state $s = (q, o_{<t})$, the advantage-weighted updates with learning rate $\eta$ satisfy:
\begin{equation}
\small
\begin{aligned}
\mathbb{E}_{a_k \sim \pi_\theta(\cdot|s)}\big[\|\Delta z(s)\|_2^2\big] 
= \eta^2\, \mathbb{E}[A^2] \Big( 1 - \sum_{k=1}^{|\mathcal{V}|} \pi_k^2 \Big) \\
\leq \eta^2\, \mathbb{E}[A^2] \Big( 1 - \exp\big(-\mathcal{H}(\pi_\theta|s)\big) \Big),
\end{aligned}
\end{equation}
where $\mathcal{H}(\pi_\theta|s) = -\sum_{v \in \mathcal{V}} \pi_\theta(v|s) \log \pi_\theta(v|s)$ is the token-level entropy and $\sum_{k=1}^{|\mathcal{V}|} \pi_k^2$ measures policy concentration. Accordingly, high-confidence tokens ($\pi_k \approx 1$) produce minimal expected update magnitudes ($\sum_{j=1}^{|\mathcal{V}|} \pi_j^2 \approx 1$), resulting in gradient attenuation for confident correct actions—an observation that constitutes the theoretical basis for gradient compensation in our method.

To effectively mitigate this gradient attenuation phenomenon for the confident correct actions, we introduce a gradient compensation factor:
\begin{equation}
\small
\beta_{i,t}^{\text{comp}} = \mathbb{I}[A_{i,t} > 0] \cdot g(\mathcal{H}_{i,t}) + \mathbb{I}[A_{i,t} \leq 0],
\end{equation}
where the compensation function $g(\cdot)$ is designed to scale inversely with the entropy:
\begin{equation}
\small
g(\mathcal{H}) = 1 + \alpha \cdot \frac{\mathcal{H}_{\max} - \mathcal{H}}{\mathcal{H}_{\max} - \mathcal{H}_{\min}}.
\end{equation}
Here, $\alpha$ determines the maximum compensation factor, while $\mathcal{H}_{\min}$ and $\mathcal{H}_{\max}$ represent the minimum and maximum token-level entropy within the current training batch. This asymmetric design targets the confidence–update-magnitude asymmetry: for positive-advantage tokens, the compensation function scales inversely with entropy to counteract gradient attenuation, moderately amplifying learning signals for high-confidence correct actions; for negative-advantage tokens, compensation is bypassed to preserve the natural penalty signals. Furthermore, for uncertain actions with high entropy, the function returns to unity, maintaining natural update dynamics.

\subsubsection{Update Magnitude Stabilization}
While gradient compensation addresses the attenuation for high-confidence actions, excessive gradient updates still may destabilize the training dynamics. By leveraging the gradient-entropy relationship established above, we utilize entropy changes as an indicator to detect such optimization instability. Specifically, our policy update decomposition reveals that changes in logits induce corresponding entropy changes through a factorized form:
\begin{equation}
\small
\Delta \mathcal{H}(\pi_\theta^k|s) \approx -\eta \sum_{a} \pi_\theta^k(a|s)^2 \cdot \Lambda_\theta^k(a|s) \cdot \xi_{i,t}(a),
\end{equation}
where $\Lambda_\theta^k(a|s)$ denotes the centered log-probability (Definition~\ref{def:centered_logprob}) that captures the deviation of action $a$'s log-probability from the policy's average entropy. $\xi_{i,t}(a) = \mathbb{I}_{\text{clip}} \cdot r_{i,t} \cdot A_{i,t}$ represents the composite update coefficient combining clipping, importance sampling, and advantage estimation. The complete derivation is provided in Appendix~\ref{sec:entropy_derivation}.

This decomposition shows that when both the policy concentration $\pi_\theta^k(a|s)^2$ and the composite coefficient magnitude $|\xi_{i,t}(a)|$ are large, entropy changes become substantial. Building upon the established gradient-entropy relationship, such large entropy changes indicate excessive gradient magnitudes that may destabilize training. Accordingly, we define the token-level instability indicator:
\begin{equation}
\small
\Xi_{i,t} = \left| \Delta \mathcal{H}(\pi_\theta^k|s_{i,t}) \right|,
\end{equation}
which quantitatively measures the estimated contribution of each token to overall update instability.

To stabilize tokens exhibiting excessive entropy changes while simultaneously maintaining learning efficiency, we formulate a stabilization factor:
\begin{equation}
\small
\beta_{i,t}^{\text{stab}} = f\left( \frac{\Xi_{i,t}}{\max_j \Xi_{j,t}} \right),
\end{equation}
where $f(\cdot)$ is a sigmoid-based decay function designed to reduce the modulation factor for tokens with large normalized entropy changes:
\begin{equation}
\small
f(x) = \lambda_{\min} + (1-\lambda_{\min}) \cdot \sigma(-\gamma(x - \tau)).
\end{equation}
In this formulation, $\sigma(z) = \frac{1}{1 + \exp(-z)}$ is the sigmoid function, $\gamma > 0$ controls the transition sharpness between stable and unstable regions, $\tau \in [0,1]$ determines the entropy change threshold triggering the decay activation, and $\lambda_{\min} = 1 - \alpha$ establishes a lower bound for the stabilization factor.

\subsubsection{Integrated Advantage Modulation}
Our complete approach integrates both compensation and stabilization mechanisms through a unified advantage modulation formulation as:
\begin{equation}
\small
A_{i,t}^{\text{final}} = A_{i,t} \cdot \beta_{i,t}^{\text{comp}} \cdot \beta_{i,t}^{\text{stab}}.
\end{equation}
Subsequently, the training objective incorporates these modulated advantages:
\begin{equation}
\small
\begin{aligned}
\mathcal{L}_{\text{DynaMO}}(\theta) 
&= \mathbb{E}_{q \sim \mathcal{D}, \{o_i\}_{i=1}^{n} \sim \pi_{\text{old}}(\cdot \mid q)} \Bigg[
\frac{1}{\sum_{i=1}^{n} |o_i|} \sum_{i=1}^{n} \sum_{t=1}^{|o_i|} \\
&\quad \min \Big( r_{i,t} A_{i,t}^{\text{final}}, 
\text{clip}(r_{i,t}, 1-\epsilon, 1+\epsilon) A_{i,t}^{\text{final}} \Big)
\Bigg],
\end{aligned}
\end{equation}
where $n$ denotes the rollout budget allocated to prompt $q$ via dynamic allocation, with the prompt subscript omitted for notational brevity.

This integrated approach addresses both challenges, leveraging the gradient-entropy relationship: the compensation mechanism maintains sufficient learning signals for high-confidence positive actions, while the stabilization mechanism prevents training instability by dampening tokens with excessive entropy changes that indicate large gradient magnitudes, with both mechanisms unified through a single hyperparameter $\alpha$ for simplified tuning.

\begin{table*}[t]
\small
\centering
{
\setlength{\tabcolsep}{3.5pt}
\renewcommand{\arraystretch}{1.23}
\caption{Comparison of benchmark results across Qwen2.5-Math-1.5B and Qwen2.5-Math-7B. Pass@K (\%) is abbreviated as P@K. The best results are bold, and the second-best results are underlined, respectively.}
\vspace{-0.3em}
\begin{tabular}{lcccccccccccccc}
\toprule
\multirow{2}{*}[-3pt]{\textbf{Method}} & \multicolumn{2}{c}{\textbf{AIME24}} & \multicolumn{2}{c}{\textbf{AIME25}} & \multicolumn{2}{c}{\textbf{AMC23}} & \multicolumn{2}{c}{\textbf{MATH500}} & \multicolumn{2}{c}{\textbf{Minerva}} & \multicolumn{2}{c}{\textbf{Olympiad}} & \multicolumn{2}{c}{\textbf{Avg.}} \\
\cmidrule(lr){2-3} \cmidrule(lr){4-5} \cmidrule(lr){6-7} \cmidrule(lr){8-9} \cmidrule(lr){10-11} \cmidrule(lr){12-13} \cmidrule(lr){14-15}
& P@1 & P@32 & P@1 & P@32 & P@1 & P@32 & P@1 & P@32 & P@1 & P@32 & P@1 & P@32 & P@1 & P@32 \\
\midrule
\multicolumn{15}{>{\columncolor{blue!15}}c}{
\textit{\textbf{Qwen2.5-Math-1.5B}}} \\
\midrule
GRPO & 13.2 & 32.3 & 7.6 & 31.5 & 56.0 & 90.0 & 54.4 & 79.2 & 17.2 & 42.8 & 25.6 & 47.0 & 29.0 & 53.8 \\
Clip-Higher & 12.4 & 34.7 & 6.4 & 30.6 & 50.6 & 89.9 & 56.8 & \underline{80.2} & 16.8 & 41.3 & \underline{26.4} & 46.8 & 28.2 & 53.9 \\
Entropy Loss & 12.6 & 33.7 & 5.8 & 28.4 & 55.6 & 86.9 & 56.3 & 78.5 & 17.6 & 43.6 & 25.4 & 46.4 & 28.9 & 52.9 \\
Fork Tokens & 9.4 & 32.0 & 5.9 & 31.4 & 52.5 & 85.6 & 54.3 & 74.2 & 16.6 & 36.8 & 25.5 & 45.2 & 27.4 & 50.9 \\
Entropy Advantages & \underline{15.7} & 35.8 & 8.9 & \underline{33.4} & 62.0 & 86.4 & \textbf{59.7} & 76.2 & \underline{18.2} & 43.0 & 25.9 & 44.9 & \underline{31.7} & 53.3 \\
Clip-COV & 13.5 & \underline{36.4} & 6.6 & \textbf{34.4} & 59.5 & 89.7 & 57.6 & 75.6 & 15.8 & \textbf{44.3} & 25.8 & \underline{47.6} & 29.8 & \underline{54.7} \\
KL-COV & 12.6 & 33.9 & \underline{9.0} & \underline{33.4} & 55.8 & \underline{91.3} & 54.2 & 78.1 & 14.8 & 40.3 & 25.4 & \textbf{48.1} & 28.6 & 54.2 \\
W-REINFORCE & 15.3 & 35.3 & 8.5 & 31.7 & \underline{63.0} & 85.7 & 56.7 & 77.7 & \underline{18.2} & 40.3 & 24.4 & 46.2 & 31.0 & 52.8 \\
\rowcolor{teal!8} \textbf{DynaMO (Ours)} & \textcolor{teal}{\textbf{17.2}} & \textcolor{teal}{\textbf{37.2}} & \textcolor{teal}{\textbf{9.8}} & 32.5 & \textcolor{teal}{\textbf{63.6}} & \textcolor{teal}{\textbf{91.9}} & \underline{58.8} & \textcolor{teal}{\textbf{81.0}} & \textcolor{teal}{\textbf{19.4}} & \underline{44.0} & \textcolor{teal}{\textbf{27.2}} & 47.1 & \textcolor{teal}{\textbf{32.7}} & \textcolor{teal}{\textbf{55.6}} \\
\midrule
\multicolumn{15}{>{\columncolor{orange!15}}c}{\textit{\textbf{Qwen2.5-Math-7B}}} \\
\midrule
GRPO & 28.8 & 52.5 & 11.7 & 34.8 & 68.3 & \underline{90.8} & 63.3 & 75.0 & 22.6 & 45.4 & 28.6 & 44.7 & 37.2 & 57.2 \\
Clip-Higher & 27.0 & 51.9 & 12.1 & 39.5 & 67.8 & 89.9 & 64.2 & \underline{83.6} & 24.0 & 46.1 & 28.1 & 46.3 & 37.2 & 59.6 \\
Entropy Loss & 30.6 & 54.6 & 13.2 & 40.6 & 66.0 & 87.0 & 60.6 & 79.6 & 23.3 & 45.9 & 30.2 & 41.1 & 37.3 & 58.1 \\
Fork Tokens & 27.1 & 52.5 & 13.4 & \underline{43.5} & 71.0 & 87.3 & \underline{65.8} & 79.3 & 26.1 & 42.4 & \underline{30.9} & \underline{47.3} & 39.1 & 58.7 \\
Entropy Advantages & 27.5 & 49.7 & 9.4 & 39.2 & 67.9 & 85.2 & 65.3 & 83.3 & 23.7 & 43.7 & 30.4 & \underline{47.3} & 37.4 & 58.1 \\
Clip-COV & 32.2 & 52.7 & 13.2 & 40.4 & \underline{72.7} & 89.3 & 64.3 & 76.8 & 25.4 & 45.9 & 29.5 & 44.6 & 39.5 & 58.3 \\
KL-COV & \underline{32.8} & 53.3 & 11.7 & 36.1 & 70.6 & 88.5 & 64.6 & 75.3 & 24.5 & 39.9 & 30.2 & 44.2 & 39.1 & 56.2 \\
W-REINFORCE & 31.8 & \underline{55.4} & \underline{14.3} & 41.0 & 72.5 & 89.8 & 64.9 & \textbf{84.0} & \underline{26.4} & \textbf{49.5} & \underline{30.9} & 46.7 & \underline{40.1} & \underline{61.1} \\
\rowcolor{teal!8} \textbf{DynaMO (Ours)} & \textcolor{teal}{\textbf{34.4}} & \textcolor{teal}{\textbf{59.0}} & \textcolor{teal}{\textbf{15.4}} & \textcolor{teal}{\textbf{46.8}} & \textcolor{teal}{\textbf{74.4}} & \textcolor{teal}{\textbf{92.9}} & \textcolor{teal}{\textbf{66.4}} & \textcolor{teal}{\textbf{84.0}} & \textcolor{teal}{\textbf{27.3}} & \underline{47.2} & \textcolor{teal}{\textbf{31.6}} & \textcolor{teal}{\textbf{50.1}} & \textcolor{teal}{\textbf{41.6}} & \textcolor{teal}{\textbf{63.3}} \\
\bottomrule
\label{tab:main_results}
\end{tabular}
}
\vspace{-2.2em}
\end{table*}

\section{Experiments}

\subsection{Experimental Setup}

\paragraph{Training Configuration.}
We conduct experiments on three different LLM scales: Qwen2.5-Math-1.5B, Qwen2.5-Math-7B, and Qwen3-14B. Our implementation builds upon the VeRL training framework~\cite{sheng2025hybridflow}, adapting it to incorporate our dual-pronged approach. Our training data is DAPO-Math-17k~\cite{yu2025dapo}, which consists of mathematical reasoning problems with verifiable integer answers, ensuring reliable reward signals for RLVR optimization. Training is performed with top-p sampling at p=1.0 and temperature set to 1.0 to maintain exploration diversity. More details are provided in Appendix~\ref{sec:training_config}.

\paragraph{Benchmarks and Metrics.}
We evaluate DynaMO on six mathematical reasoning benchmarks: AIME24, AIME25~\cite{AIME}, AMC23~\cite{AMC}, MATH500~\cite{hendrycks2021measuring}, Minerva~\cite{lewkowycz2022solving}, and Olympiad~\cite{he2024olympiadbench}, covering various problem difficulties and types. Detailed descriptions of these benchmarks are presented in Appendix~\ref{sec:benchmark_details}.

We report average Pass@1 and Pass@32 across three independent training runs with different random seeds, using Math-Verify~\cite{mathverify} for answer verification and top-p=1.0, temperature=1.0 for diverse generation. Following \citet{RLLimit}, Pass@K equals 1 if at least one of K sampled outputs passes verification, with unbiased estimation adopted from \citet{unbiasedpassk} to mitigate evaluation variance.

\paragraph{Baseline Methods.}
We compare DynaMO with strong baselines, including standard GRPO~\cite{shao2024deepseekmath} and various entropy intervention techniques: Clip-Higher~\cite{yu2025dapo}, Entropy Loss~\cite{cheng2025reasoning}, Fork Tokens~\cite{wang2025beyond}, Entropy Advantages~\cite{cheng2025reasoning}, coverage-based methods (Clip-COV, KL-COV)~\cite{zhu2025surprising}, and W-REINFORCE~\cite{tan2025gtpo}. All experiments use consistent hyperparameters for fair comparison.

\begin{table*}[t]
\small
\centering
\setlength{\tabcolsep}{9.7pt}
\renewcommand{\arraystretch}{1.1}
\caption{Ablation study on Qwen2.5-Math-7B with Pass@1 (\%). DRA: Dynamic Rollout Allocation, UMS: Update Magnitude Stabilization, GC: Gradient Compensation, w/o ALL: standard GRPO}
\vspace{-0.5em}
\label{tab:ablation_overall}
\begin{tabular}{lccccccc}
\toprule
\textbf{Method} & \textbf{AIME24} & \textbf{AIME25} & \textbf{AMC23} & \textbf{MATH500} & \textbf{Minerva} & \textbf{Olympiad} & \textbf{Avg.} \\
\midrule
\rowcolor{teal!8} \textbf{DynaMO} & \textcolor{teal}{\textbf{34.4}} & \textcolor{teal}{\textbf{15.4}} & \textcolor{teal}{\textbf{74.4}} & \textcolor{teal}{\textbf{66.4}} & \textcolor{teal}{\textbf{27.3}} & \textcolor{teal}{\textbf{31.6}} & \textcolor{teal}{\textbf{41.6}} \\
\midrule
- $w/o$ GC & 33.8 & 15.0 & 71.9 & 65.0 & 26.7 & 29.3 & 40.3 \\
- $w/o$ UMS & 33.2 & 14.7 & 71.9 & 64.9 & 25.9 & 30.2 & 40.1 \\
- $w/o$ DRA & 31.9 & 15.2 & 73.4 & 65.7 & 23.0 & 30.4 & 39.9 \\
\midrule
- $w/o$ GC \& DRA & 31.9 & 14.5 & 69.0 & 64.4 & 23.7 & 29.7 & 38.9 \\
- $w/o$ GC \& UMS & 30.5 & 14.3 & 70.2 & 63.5 & 22.2 & 29.4 & 38.4 \\
- $w/o$ UMS \& DRA & 30.0 & 13.8 & 70.1 & 61.6 & 19.9 & 30.2 & 37.6 \\
\midrule
- $w/o$ ALL  & 28.8 & 11.7 & 68.3 & 63.3 & 22.6 & 28.6 & 37.2 \\
\bottomrule
\end{tabular}
\vspace{-1.3em}
\end{table*}

\subsection{Main Results}
Table~\ref{tab:main_results} presents the comprehensive comparison results across six mathematical reasoning benchmarks on Qwen2.5-Math-1.5B and Qwen2.5-Math-7B. DynaMO consistently outperforms all baseline methods by effectively addressing two complementary challenges: the variance-driven rollout allocation concentrates computational budget on problems with balanced success-failure distributions where Bernoulli variance signals high gradient informativeness, while the gradient-aware advantage modulation provides compensation for gradient attenuation in high-confidence correct actions and stabilization against excessive update magnitudes signaled by large entropy changes. Furthermore, a comparison with entropy intervention baselines reveals their critical limitations: coarse-grained clipping and sequence-level reweighting methods like Clip-Higher, Clip-COV, and KL-COV lack fine-grained control over token-level dynamics, while entropy-induced advantage methods such as Entropy Advantages and W-REINFORCE introduce training instability without principled stabilization mechanisms. Results on Qwen3-14B demonstrating effective scaling are presented in Section~\ref{sec:scaling}.

\begin{figure}[t]
\vspace{-0.5em}
\subfigure{
\begin{minipage}[t]{0.48\linewidth}
\centerline{\includegraphics[width=1\hsize]{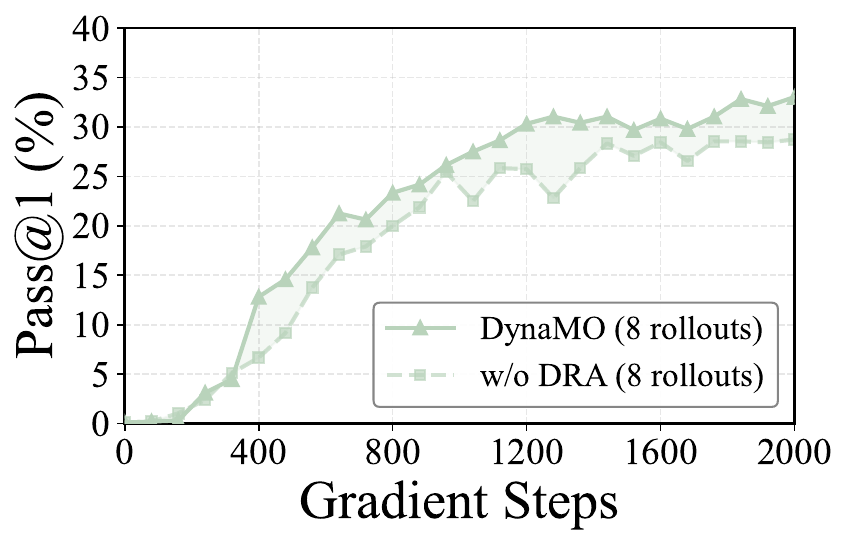}}
\vspace{4pt} 
\centerline{\includegraphics[width=1\hsize]{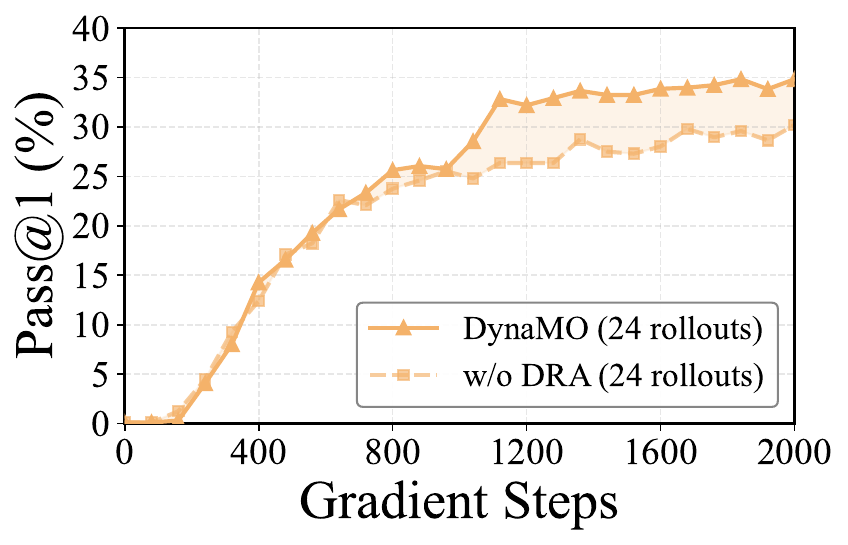}}
\vspace{4pt} 
\end{minipage}}
\subfigure{
\begin{minipage}[t]{0.48\linewidth}
\centerline{\includegraphics[width=1\hsize]{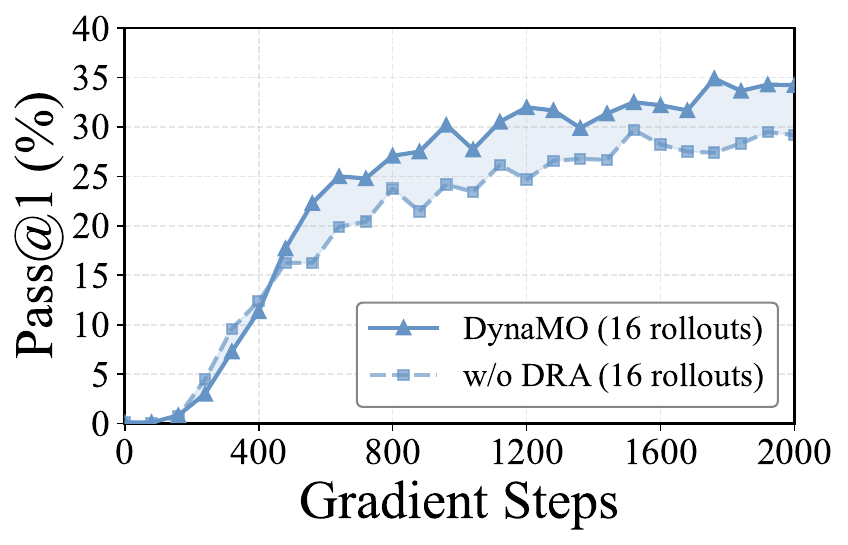}}
\vspace{4pt} 
\centerline{\includegraphics[width=1\hsize]{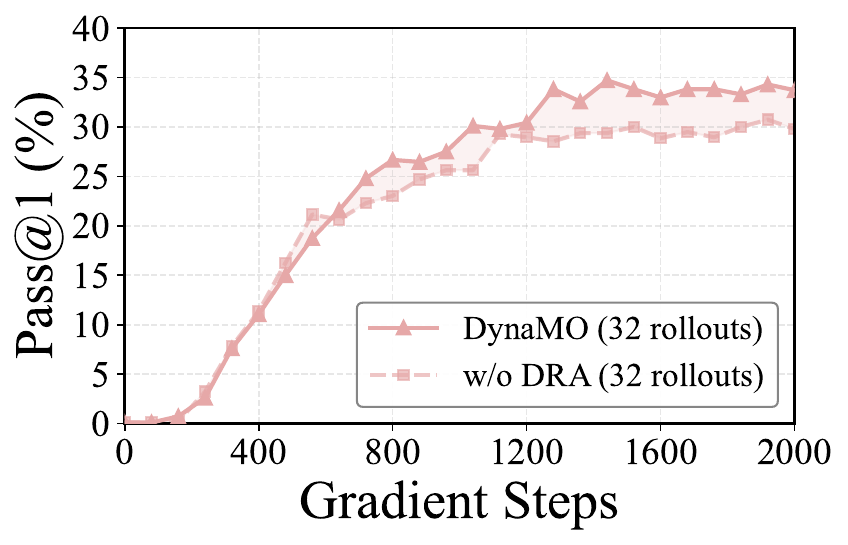}}
\vspace{4pt} 
\end{minipage}}
\vspace{-1.5em}
\caption{Impact of DRA across different computational budgets on AIME24 with average Pass@1 (\%). Solid lines denote the full DynaMO, and dashed lines are the variant that substitutes DRA with uniform allocation.}
\label{fig:budget_sensitivity}
\vspace{-1.1em}
\end{figure}

\subsection{Ablation Study}

\subsubsection{Overall Component Analysis}
Table~\ref{tab:ablation_overall} conducts systematic ablation on Qwen2.5-Math-7B. DynaMO integrates three components: Dynamic Rollout Allocation (DRA), Update Magnitude Stabilization (UMS), and Gradient Compensation (GC). Removing individual components degrades average performance, with DRA showing the largest impact on Minerva, GC on Olympiad, and UMS providing stability across benchmarks. Removing multiple components reveals synergistic effects: \textit{w/o} GC \& UMS degrades more than the sum of individual removals, indicating UMS stabilizes the amplified gradients from GC.

\begin{figure}[t]
\centerline{\includegraphics[width=0.95\linewidth]{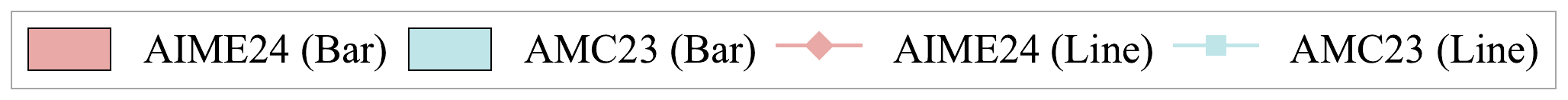}}
\vspace{-5pt}  
\subfigure{
\begin{minipage}[t]{0.48\linewidth}
\centerline{\includegraphics[width=1\hsize]{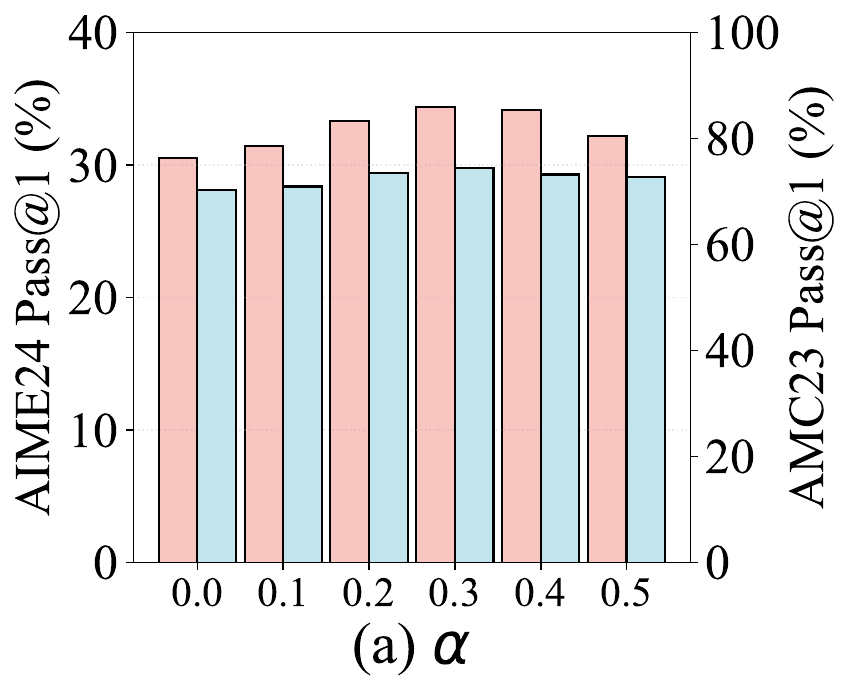}}
\vspace{1pt} 
\centerline{\includegraphics[width=1\hsize]{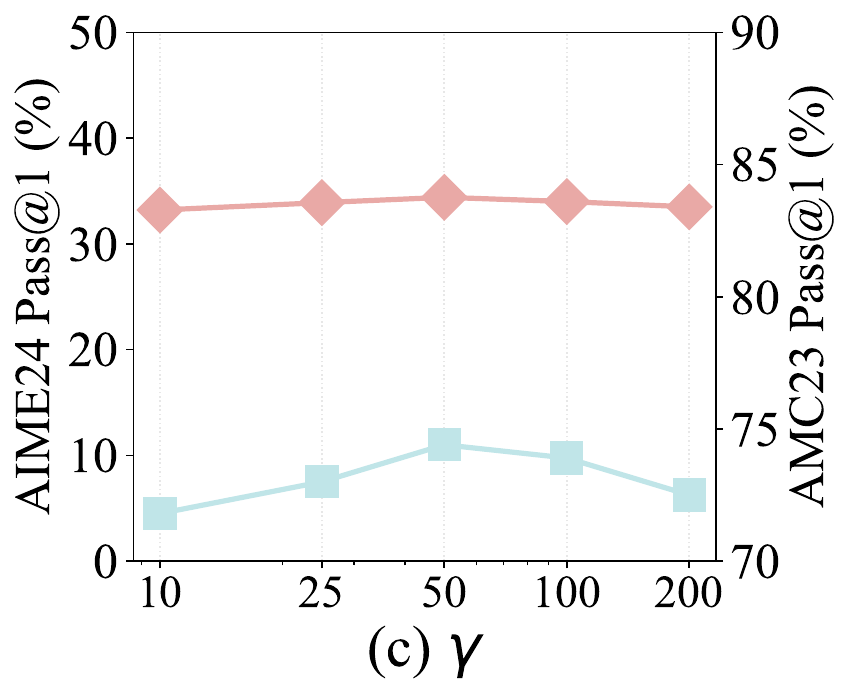}}
\vspace{1pt} 
\end{minipage}}
\subfigure{
\begin{minipage}[t]{0.48\linewidth}
\centerline{\includegraphics[width=1\hsize]{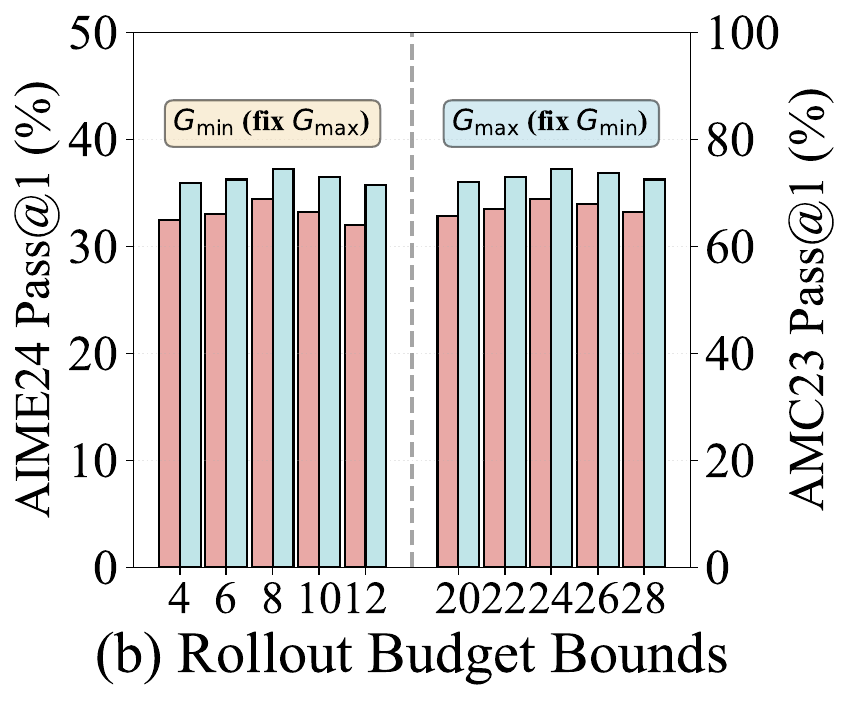}}
\vspace{2pt} 
\centerline{\includegraphics[width=1\hsize]{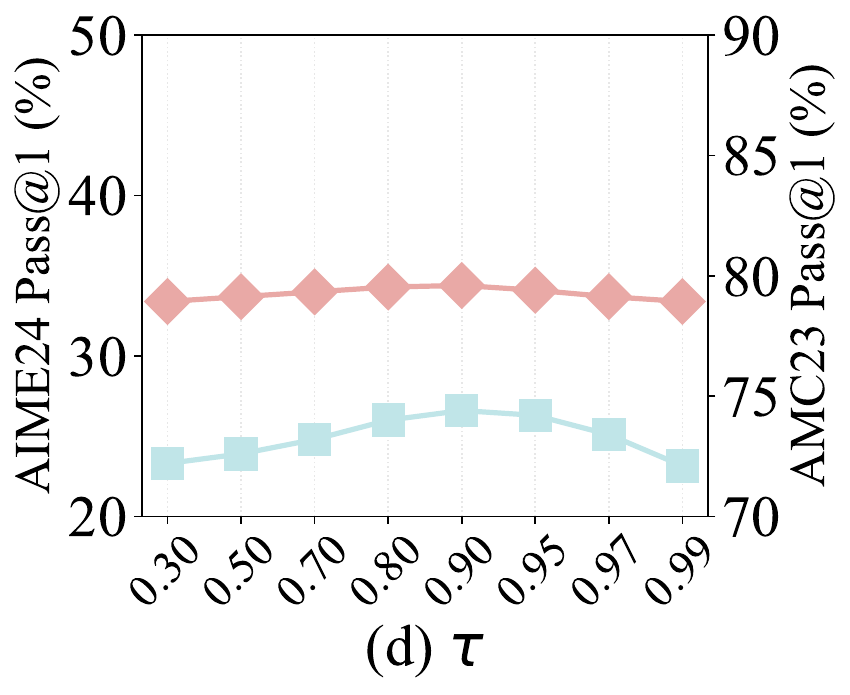}}
\vspace{2pt} 
\end{minipage}}
\vspace{-1em}
\caption{Hyperparameter sensitivity analysis on Qwen2.5-Math-7B across AIME24 and AMC23.}
\vspace{-1.1em}
\label{fig:hyperparameter_analysis}
\end{figure}

\subsubsection{Impact of Dynamic Rollout Allocation}
Figure~\ref{fig:budget_sensitivity} evaluates DRA across varying computational budgets, specifically ranging from an average of 8 to 32 rollouts per problem. The results demonstrate that DRA provides stable performance gains across all configurations. During early training, limited historical statistics result in similar allocations across problems. However, as variance data accumulates, DRA progressively concentrates the budget on problems with balanced success-failure distributions where Bernoulli variance peaks. Crucially, these problems reside within the capability gap, being neither trivially solved nor currently inaccessible, where both positive and negative advantage signals are actively generated, thereby maximizing learning signal quality per computational unit. In contrast, uniform allocation continues wasting resources on problems outside this optimal learning zone throughout training.

\subsection{Hyperparameter Analysis}
Figure~\ref{fig:hyperparameter_analysis} illustrates the hyperparameter sensitivity analysis on Qwen2.5-Math-7B across AIME24 and AMC23. The unified modulation parameter $\alpha$ exhibits an inverted-U pattern: performance degrades without modulation at $\alpha=0$, peaks at moderate values, and subsequently declines at excessive settings, validating that insufficient modulation fails to address gradient attenuation while over-intervention constrains the learning process. Allocation bounds $G_{\min}$ and $G_{\max}$ demonstrate stable performance across wide ranges, confirming variance-driven allocation avoids under-sampling issues with minimal tuning. Similarly, the stabilization sharpness $\gamma$ exhibits a flat plateau around its optimal values, balancing selective intervention against learning flexibility. Entropy threshold $\tau$ displays smooth variation characterized by a broad optimal region, thereby demonstrating the reliability in identifying problematic tokens.

\subsection{Scaling to Larger Models}
\label{sec:scaling}
\begin{figure}[t]
\centering
\includegraphics[width=1\columnwidth]{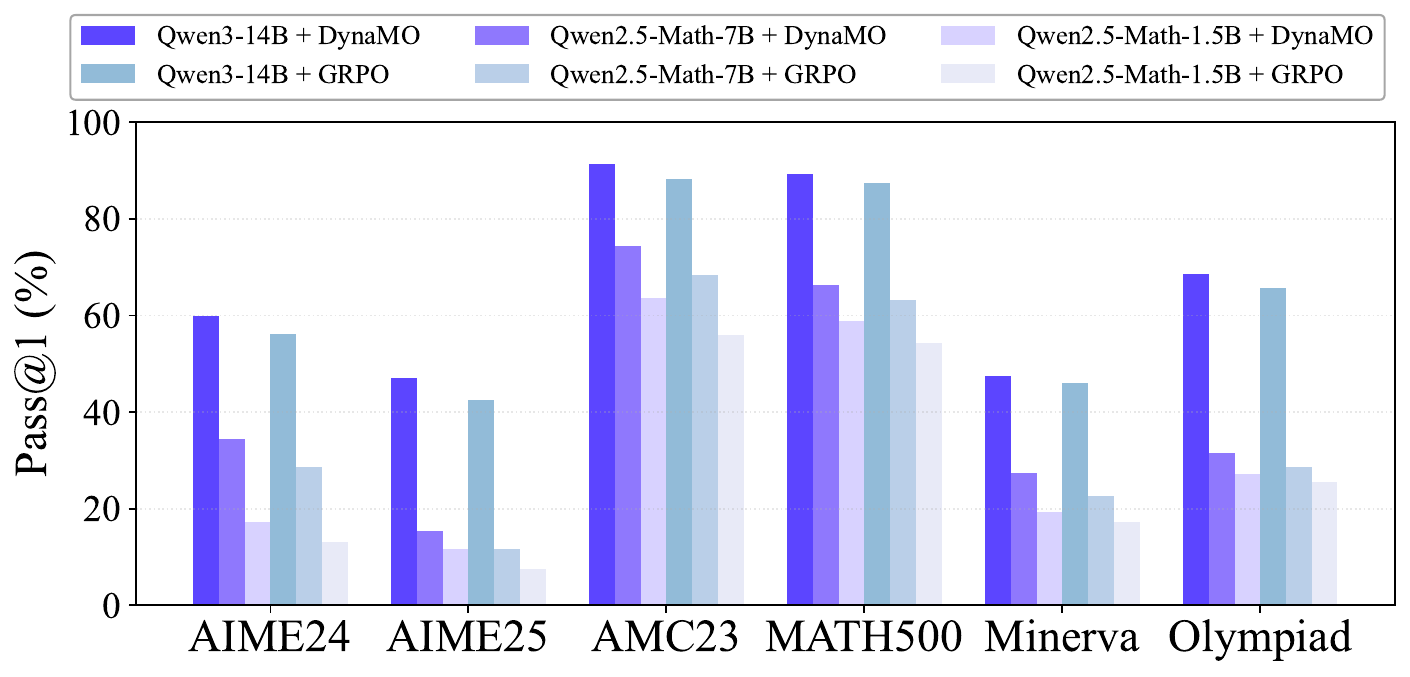}
\caption{Performance comparison across different LLM scales (1.5B/7B/14B) with DynaMO and GRPO.}
\vspace{-1em}
\label{fig:scaling_analysis}
\end{figure}

To validate scalability, we extend experiments to Qwen3-14B across all benchmarks in Figure~\ref{fig:scaling_analysis}. DynaMO consistently outperforms GRPO with widening gaps as scale increases: moderate gains at 1.5B, amplified substantially at 7B, and expanded further at 14B. This trend confirms that our theoretically-grounded mechanisms become increasingly effective at larger scales, as variance-driven allocation and gradient-aware modulation better exploit the enhanced representational capacity while mitigating intensified optimization instabilities.

\section{Case Study}

\begin{figure}[t]
\centerline{\includegraphics[width=0.45\linewidth]{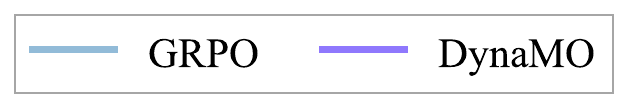}}
\vspace{-5pt}  
\subfigure{
\begin{minipage}[t]{0.48\linewidth}
\centerline{\includegraphics[width=1\hsize]{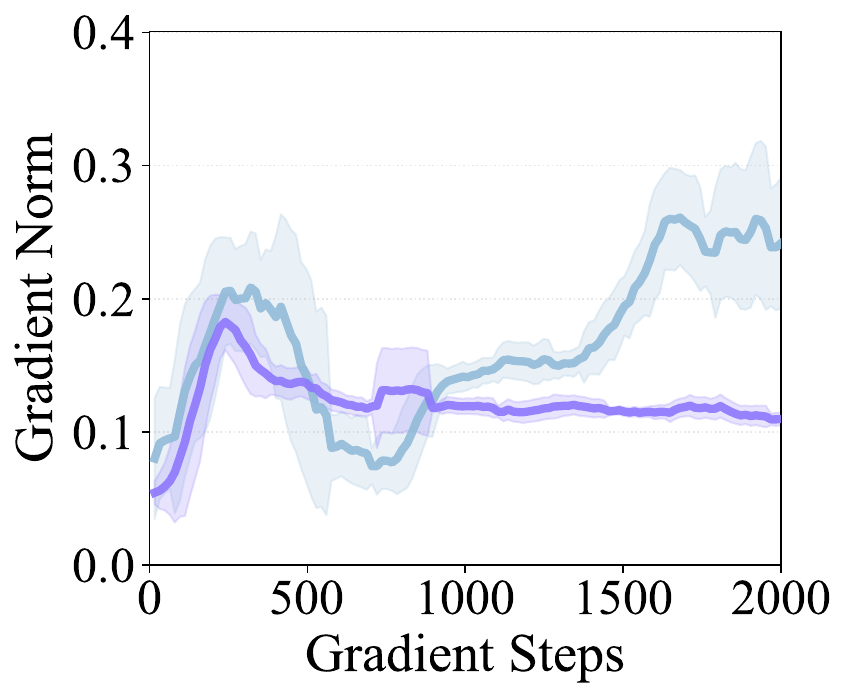}}
\end{minipage}}
\subfigure{
\begin{minipage}[t]{0.48\linewidth}
\centerline{\includegraphics[width=1\hsize]{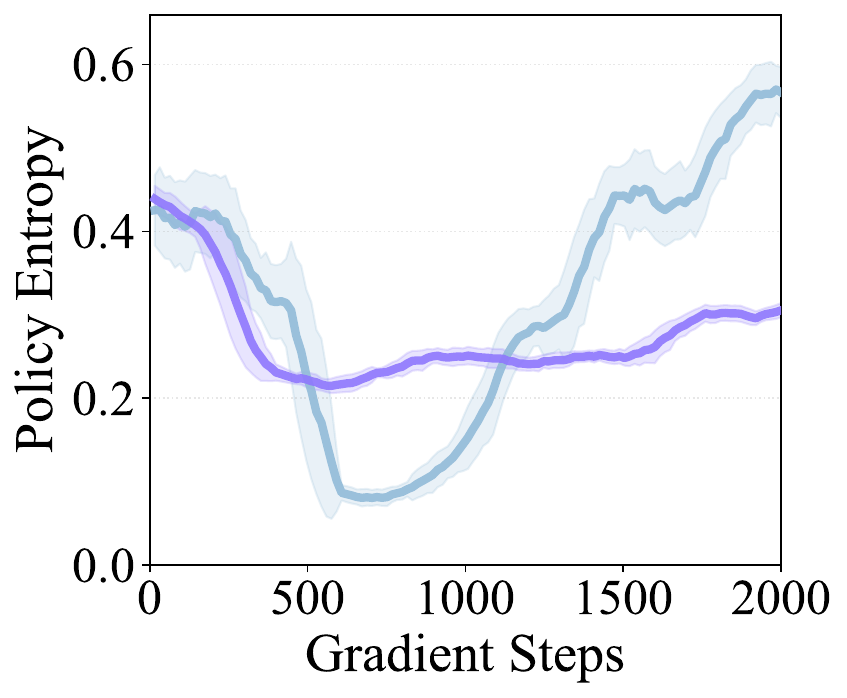}}
\end{minipage}}
\vspace{-1em}
\caption{Training dynamics comparison. DynaMO maintains stable gradient norms and smooth entropy evolution, while GRPO exhibits severe spikes and fluctuations. Shaded regions show raw data range.}
\vspace{-0.7em}
\label{fig:case_study}
\end{figure}

To provide a deeper insight into the effectiveness of DynaMO, we compare it with GRPO through training dynamics analysis. As visually depicted in Figure~\ref{fig:case_study}, GRPO suffers from severe gradient spikes and erratic entropy fluctuations, whereas DynaMO maintains stable dynamics across both metrics. Two core components of DynaMO contribute to this stability: dynamic rollout allocation concentrates the budget on high-variance problems, while gradient-aware modulation prevents excessive update magnitudes. Notably, despite following similar overall trends, DynaMO achieves substantially smoother transitions in gradient norms and policy entropy, indicating more controlled optimization process throughout training. These phenomena confirm that DynaMO mitigates training instability at both sequence and token levels while simultaneously achieving superior performance.

\section{Conclusion}
In this paper, we propose \textbf{DynaMO}, a theoretically-grounded dual-pronged framework designed to systematically address fundamental RLVR challenges. \textit{At the sequence level}, we prove that uniform allocation is suboptimal and subsequently derive a variance-minimizing allocation strategy to concentrate computational resources on high-informativeness problems. \textit{At the token level}, we establish the gradient-entropy relationship, enabling an integrated advantage modulation mechanism that compensates for gradient attenuation while stabilizing excessive updates. Extensive experiments conducted across multiple reasoning benchmarks and varying LLM scales demonstrate consistent improvements over strong baselines, with comprehensive ablation studies validating the independent contributions from both mechanisms.

\section*{Limitations}
Our evaluation is conducted on the Qwen model family, spanning three distinct parameter scales (1.5B, 7B, and 14B). Although we have not yet systematically evaluated other model families, it is important to note that our method operates purely at the algorithmic level—modifying rollout allocation and advantage computation—without relying on any architecture-specific features, suggesting broad transferability potential. Future work could extend this evaluation to additional model architectures and scales, as well as multimodal reasoning scenarios where robustness against adversarial attacks~\cite{li2026makingmllmsblindadversarial} becomes critical.

\bibliography{custom}

@inproceedings{mnih2016asynchronous,
  title={Asynchronous methods for deep reinforcement learning},
  author={Mnih, Volodymyr and Badia, Adria Puigdomenech and Mirza, Mehdi and Graves, Alex and Lillicrap, Timothy and Harley, Tim and Silver, David and Kavukcuoglu, Koray},
  booktitle={International Conference on Machine Learning (ICML)},
  year={2016}
}

@article{schulman2017proximal,
  title={Proximal policy optimization algorithms},
  author={Schulman, John and Wolski, Filip and Dhariwal, Prafulla and Radford, Alec and Klimov, Oleg},
  journal={arXiv preprint arXiv:1707.06347},
  year={2017}
}

@article{sutton2018reinforcement,
  title={Reinforcement learning: an introduction},
  author={Sutton, Richard and Barto, Andrew},
  year={2018}
}

@article{li2025logit,
  title={Logit dynamics in softmax policy gradient methods},
  author={Li, Yingru},
  journal={arXiv preprint arXiv:2506.12912},
  year={2025}
}

@article{ouyang2022training,
  title={Training language models to follow instructions with human feedback},
  author={Ouyang, Long and Wu, Jeffrey and Jiang, Xu and Almeida, Diogo and Wainwright, Carroll and Mishkin, Pamela and Zhang, Chong and Agarwal, Sandhini and Slama, Katarina and Ray, Alex and others},
  journal={Conference on Neural Information Processing Systems (NeurIPS)},
  year={2022}
}

@article{bai2022training,
  title={Training a helpful and harmless assistant with reinforcement learning from human feedback},
  author={Bai, Yuntao and Jones, Andy and Ndousse, Kamal and Askell, Amanda and Chen, Anna and DasSarma, Nova and Drain, Dawn and Fort, Stanislav and Ganguli, Deep and Henighan, Tom and others},
  journal={arXiv preprint arXiv:2204.05862},
  year={2022}
}

@article{shao2024deepseekmath,
  title={Deepseekmath: Pushing the limits of mathematical reasoning in open language models},
  author={Shao, Zhihong and Wang, Peiyi and Zhu, Qihao and Xu, Runxin and Song, Junxiao and Bi, Xiao and Zhang, Haowei and Zhang, Mingchuan and Li, YK and Wu, Yang and others},
  journal={arXiv preprint arXiv:2402.03300},
  year={2024}
}

@article{deepseekai2025deepseekr1incentivizingreasoningcapability,
  title={Deepseek-r1: Incentivizing reasoning capability in llms via reinforcement learning},
  author={Guo, Daya and Yang, Dejian and Zhang, Haowei and Song, Junxiao and Zhang, Ruoyu and Xu, Runxin and Zhu, Qihao and Ma, Shirong and Wang, Peiyi and Bi, Xiao and others},
  journal={arXiv preprint arXiv:2501.12948},
  year={2025}
}

@article{jaech2024openai,
  title={Openai o1 system card},
  author={Jaech, Aaron and Kalai, Adam and Lerer, Adam and Richardson, Adam and El-Kishky, Ahmed and Low, Aiden and Helyar, Alec and Madry, Aleksander and Beutel, Alex and Carney, Alex and others},
  journal={arXiv preprint arXiv:2412.16720},
  year={2024}
}

@article{team2025kimi,
  title={Kimi k1. 5: Scaling reinforcement learning with llms},
  author={Team, Kimi and Du, Angang and Gao, Bofei and Xing, Bowei and Jiang, Changjiu and Chen, Cheng and Li, Cheng and Xiao, Chenjun and Du, Chenzhuang and Liao, Chonghua and others},
  journal={arXiv preprint arXiv:2501.12599},
  year={2025}
}

@article{liu2025understanding,
  title={Understanding r1-zero-like training: A critical perspective},
  author={Liu, Zichen and Chen, Changyu and Li, Wenjun and Qi, Penghui and Pang, Tianyu and Du, Chao and Lee, Wee Sun and Lin, Min},
  journal={arXiv preprint arXiv:2503.20783},
  year={2025}
}

@article{yu2025dapo,
  title={Dapo: An open-source llm reinforcement learning system at scale},
  author={Yu, Qiying and Zhang, Zheng and Zhu, Ruofei and Yuan, Yufeng and Zuo, Xiaochen and Yue, Yu and Fan, Tiantian and Liu, Gaohong and Liu, Lingjun and Liu, Xin and others},
  journal={arXiv preprint arXiv:2503.14476},
  year={2025}
}

@article{chu2025gpg,
  title={Gpg: A simple and strong reinforcement learning baseline for model reasoning},
  author={Chu, Xiangxiang and Huang, Hailang and Zhang, Xiao and Wei, Fei and Wang, Yong},
  journal={arXiv preprint arXiv:2504.02546},
  year={2025}
}

@article{hu2025open,
  title={Open-reasoner-zero: An open source approach to scaling up reinforcement learning on the base model},
  author={Hu, Jingcheng and Zhang, Yinmin and Han, Qi and Jiang, Daxin and Zhang, Xiangyu and Shum, Heung-Yeung},
  journal={arXiv preprint arXiv:2503.24290},
  year={2025}
}

@inproceedings{haarnoja2018soft,
  title={Soft actor-critic: Off-policy maximum entropy deep reinforcement learning with a stochastic actor},
  author={Haarnoja, Tuomas and Zhou, Aurick and Abbeel, Pieter and Levine, Sergey},
  booktitle={International Conference on Machine Learning (ICML)},
  year={2018}
}

@article{luo2025deepscaler,
  title={Deepscaler: Surpassing o1-preview with a 1.5 b model by scaling rl},
  author={Luo, Michael and Tan, Sijun and Wong, Justin and Shi, Xiaoxiang and Tang, William Y and Roongta, Manan and Cai, Colin and Luo, Jeffrey and Zhang, Tianjun and Li, Li Erran and others},
  journal={Notion Blog},
  year={2025}
}

@article{yang2025dcpo,
  title={DCPO: Dynamic clipping policy optimization},
  author={Yang, Shihui and Dou, Chengfeng and Guo, Peidong and Lu, Kai and Ju, Qiang and Deng, Fei and Xin, Rihui},
  journal={arXiv preprint arXiv:2509.02333},
  year={2025}
}

@article{zhu2025surprising,
  title={The surprising effectiveness of negative reinforcement in LLM reasoning},
  author={Zhu, Xinyu and Xia, Mengzhou and Wei, Zhepei and Chen, Wei-Lin and Chen, Danqi and Meng, Yu},
  journal={arXiv preprint arXiv:2506.01347},
  year={2025}
}

@article{cheng2025reasoning,
  title={Reasoning with exploration: An entropy perspective},
  author={Cheng, Daixuan and Huang, Shaohan and Zhu, Xuekai and Dai, Bo and Zhao, Wayne Xin and Zhang, Zhenliang and Wei, Furu},
  journal={arXiv preprint arXiv:2506.14758},
  year={2025}
}

@article{tan2025gtpo,
  title={GTPO and GRPO-S: Token and sequence-level reward shaping with policy entropy},
  author={Tan, Hongze and Pan, Jianfei},
  journal={arXiv preprint arXiv:2508.04349},
  year={2025}
}

@article{wang2025beyond,
  title={Beyond the 80/20 rule: High-entropy minority tokens drive effective reinforcement learning for llm reasoning},
  author={Wang, Shenzhi and Yu, Le and Gao, Chang and Zheng, Chujie and Liu, Shixuan and Lu, Rui and Dang, Kai and Chen, Xionghui and Yang, Jianxin and Zhang, Zhenru and others},
  journal={arXiv preprint arXiv:2506.01939},
  year={2025}
}

@article{yao2025optimizing,
  title={Optimizing chain-of-thought reasoners via gradient variance minimization in rejection sampling and rl},
  author={Yao, Jiarui and Hao, Yifan and Zhang, Hanning and Dong, Hanze and Xiong, Wei and Jiang, Nan and Zhang, Tong},
  journal={arXiv preprint arXiv:2505.02391},
  year={2025}
}

@article{hendrycks2021measuring,
  title={Measuring mathematical problem solving with the math dataset},
  author={Hendrycks, Dan and Burns, Collin and Kadavath, Saurav and Arora, Akul and Basart, Steven and Tang, Eric and Song, Dawn and Steinhardt, Jacob},
  journal={arXiv preprint arXiv:2103.03874},
  year={2021}
}

@article{lewkowycz2022solving,
  title={Solving quantitative reasoning problems with language models},
  author={Lewkowycz, Aitor and Andreassen, Anders and Dohan, David and Dyer, Ethan and Michalewski, Henryk and Ramasesh, Vinay and Slone, Ambrose and Anil, Cem and Schlag, Imanol and Gutman-Solo, Theo and others},
  journal={Conference on Neural Information Processing Systems (NeurIPS)},
  year={2022}
}

@article{he2024olympiadbench,
  title={Olympiadbench: A challenging benchmark for promoting agi with olympiad-level bilingual multimodal scientific problems},
  author={He, Chaoqun and Luo, Renjie and Bai, Yuzhuo and Hu, Shengding and Thai, Zhen Leng and Shen, Junhao and Hu, Jinyi and Han, Xu and Huang, Yujie and Zhang, Yuxiang and others},
  journal={arXiv preprint arXiv:2402.14008},
  year={2024}
}

@article{unbiasedpassk,
  title={Evaluating large language models trained on code},
  author={Chen, Mark and Tworek, Jerry and Jun, Heewoo and Yuan, Qiming and Pinto, Henrique Ponde De Oliveira and Kaplan, Jared and Edwards, Harri and Burda, Yuri and Joseph, Nicholas and Brockman, Greg and others},
  journal={arXiv preprint arXiv:2107.03374},
  year={2021}
}

@article{RLLimit,
  title={Does reinforcement learning really incentivize reasoning capacity in llms beyond the base model?},
  author={Yue, Yang and Chen, Zhiqi and Lu, Rui and Zhao, Andrew and Wang, Zhaokai and Song, Shiji and Huang, Gao},
  journal={arXiv preprint arXiv:2504.13837},
  year={2025}
}

@article{gulcehre2023reinforced,
  title={Reinforced self-training (rest) for language modeling},
  author={Gulcehre, Caglar and Paine, Tom Le and Srinivasan, Srivatsan and Konyushkova, Ksenia and Weerts, Lotte and Sharma, Abhishek and Siddhant, Aditya and Ahern, Alex and Wang, Miaosen and Gu, Chenjie and others},
  journal={arXiv preprint arXiv:2308.08998},
  year={2023}
}

@article{dong2023raft,
  title={Raft: Reward ranked finetuning for generative foundation model alignment},
  author={Dong, Hanze and Xiong, Wei and Goyal, Deepanshu and Zhang, Yihan and Chow, Winnie and Pan, Rui and Diao, Shizhe and Zhang, Jipeng and Shum, Kashun and Zhang, Tong},
  journal={arXiv preprint arXiv:2304.06767},
  year={2023}
}

@article{tong2024dart,
  title={Dart-math: Difficulty-aware rejection tuning for mathematical problem-solving},
  author={Tong, Yuxuan and Zhang, Xiwen and Wang, Rui and Wu, Ruidong and He, Junxian},
  journal={Conference on Neural Information Processing Systems (NeurIPS)},
  year={2024}
}

@inproceedings{bengio2009curriculum,
  title={Curriculum learning},
  author={Bengio, Yoshua and Louradour, J{\'e}r{\^o}me and Collobert, Ronan and Weston, Jason},
  booktitle={International Conference on Machine Learning (ICML)},
  year={2009}
}

@article{schaul2015prioritized,
  title={Prioritized experience replay},
  author={Schaul, Tom and Quan, John and Antonoglou, Ioannis and Silver, David},
  journal={arXiv preprint arXiv:1511.05952},
  year={2015}
}

@misc{AIME,
  title={American Invitational Mathematics Examination - AIME},
  author={{MAA}},
  url={https://maa.org/},
  year={2025}
}

@misc{AMC,
  title={American Mathematics Competitions - AMC},
  author={{MAA}},
  url={https://maa.org/},
  year={2023}
}

@misc{mathverify,
  title={Math-Verify},
  author={HuggingFace},
  url={https://github.com/huggingface/Math-Verify},
  year={2025}
}

@inproceedings{sheng2025hybridflow,
  title={Hybridflow: A flexible and efficient rlhf framework},
  author={Sheng, Guangming and Zhang, Chi and Ye, Zilingfeng and Wu, Xibin and Zhang, Wang and Zhang, Ru and Peng, Yanghua and Lin, Haibin and Wu, Chuan},
  booktitle={European Conference on Computer Systems (EuroSys)},
  year={2025}
}

@inproceedings{liu2021quadrupletbert,
  title={QuadrupletBERT: An efficient model for embedding-based large-scale retrieval},
  author={Liu, Peiyang and Wang, Sen and Wang, Xi and Ye, Wei and Zhang, Shikun},
  booktitle={Proceedings of the 2021 Conference of the North American Chapter of the Association for Computational Linguistics: Human Language Technologies},
  pages={3734--3739},
  year={2021}
}

@inproceedings{liu2021improving,
  title={Improving embedding-based large-scale retrieval via label enhancement},
  author={Liu, Peiyang and Wang, Xi and Wang, Sen and Ye, Wei and Xi, Xiangyu and Zhang, Shikun},
  booktitle={Findings of the Association for Computational Linguistics: EMNLP 2021},
  pages={133--142},
  year={2021}
}

@inproceedings{liu2022label,
  title={Label smoothing for text mining},
  author={Liu, Peiyang and Xi, Xiangyu and Ye, Wei and Zhang, Shikun},
  booktitle={Proceedings of the 29th international conference on computational linguistics},
  pages={2210--2219},
  year={2022}
}

@article{fang2026ckpro,
  author       = {Tianqing Fang and
                  Zhisong Zhang and
                  Xiaoyang Wang and
                  Rui Wang and
                  Can Qin and
                  Yuxuan Wan and
                  Jun{-}Yu Ma and
                  Ce Zhang and
                  Jiaqi Chen and
                  Xiyun Li and
                  Hongming Zhang and
                  Haitao Mi and
                  Dong Yu},
  title        = {Cognitive Kernel-Pro: {A} Framework for Deep Research Agents and Agent
                  Foundation Models Training},
  journal      = {CoRR},
  volume       = {abs/2508.00414},
  year         = {2025},
  url          = {https://doi.org/10.48550/arXiv.2508.00414},
  doi          = {10.48550/ARXIV.2508.00414},
  eprinttype   = {arXiv},
  eprint       = {2508.00414},
}

@article{li2026cso,
  author       = {Mukai Li and
                  Qingcheng Zeng and
                  Tianqing Fang and
                  Zhenwen Liang and
                  Linfeng Song and
                  Qi Liu and
                  Haitao Mi and
                  Dong Yu},
  title        = {Verified Critical Step Optimization for {LLM} Agents},
  journal      = {CoRR},
  volume       = {abs/2602.03412},
  year         = {2026},
  url          = {https://doi.org/10.48550/arXiv.2602.03412},
  doi          = {10.48550/ARXIV.2602.03412},
  eprinttype   = {arXiv},
  eprint       = {2602.03412},
}

@article{yu2026rfew,
  author       = {Wenhao Yu and
                  Zhenwen Liang and
                  Chengsong Huang and
                  Kishan Panaganti and
                  Tianqing Fang and
                  Haitao Mi and
                  Dong Yu},
  title        = {Guided Self-Evolving LLMs with Minimal Human Supervision},
  journal      = {CoRR},
  volume       = {abs/2512.02472},
  year         = {2025},
  url          = {https://doi.org/10.48550/arXiv.2512.02472},
  doi          = {10.48550/ARXIV.2512.02472},
  eprinttype   = {arXiv},
  eprint       = {2512.02472},
}

@misc{fang2026placingpuzzlepiecesmatter,
      title={Placing Puzzle Pieces Where They Matter: A Question Augmentation Framework for Reinforcement Learning}, 
      author={Yangyi Fang and Jiaye Lin and Xiaoliang Fu and Cong Qin and Haolin Shi},
      year={2026},
      eprint={2604.15830},
      archivePrefix={arXiv},
      primaryClass={cs.LG},
      url={https://arxiv.org/abs/2604.15830}, 
}

@misc{li2026makingmllmsblindadversarial,
      title={Making MLLMs Blind: Adversarial Smuggling Attacks in MLLM Content Moderation}, 
      author={Zhiheng Li and Zongyang Ma and Yuntong Pan and Ziqi Zhang and Xiaolei Lv and Bo Li and Jun Gao and Jianing Zhang and Chunfeng Yuan and Bing Li and Weiming Hu},
      year={2026},
      eprint={2604.06950},
      archivePrefix={arXiv},
      primaryClass={cs.CV},
      url={https://arxiv.org/abs/2604.06950 }, 
}

\appendix

\label{sec:appendix}
\section{Derivation of Entropy Change Estimation for Gradient-Aware Advantage Modulation}
\label{sec:entropy_derivation}

We derive the first-order entropy change estimation for tabular softmax policy through a factorized decomposition approach. Let the actor policy $\pi_\theta$ be a tabular softmax policy where each state-action pair $(s, a)$ is associated with an individual logit parameter $z_{s,a} = \theta_{s,a}$.

\subsection{Preliminary Definitions}

\begin{definition}[Policy Concentration Measure]
\label{def:concentration}
For a policy $\pi_\theta(\cdot|s)$ at state $s$, we define the \textit{concentration coefficient}:
\begin{equation}
\small
\Phi(\pi_\theta|s) := \sum_{a} \pi_\theta(a|s)^2.
\end{equation}
This measures the degree of policy concentration, with higher values indicating a more deterministic policy.
\end{definition}

\begin{definition}[Entropy-Weighted Log-Probability]
\label{def:centered_logprob}
We define the \textit{centered log-probability} for action $a$ under policy $\pi_\theta$ at state $s$:
\begin{equation}
\small
\begin{aligned}
\Lambda_\theta(a|s) := \log\pi_\theta(a|s) + \mathcal{H}(\pi_\theta|s).
\end{aligned}
\end{equation}
This represents the deviation of action $a$'s log-probability from the policy's average entropy, capturing the information-theoretic distance from uniform randomness.
\end{definition}

\subsection{Entropy Gradient Derivation}
We use Taylor's expansion under first-order approximation:
\begin{equation}
\small
\label{eq:taylor_entropy}
\begin{aligned}  
    &\mathcal{H}(\pi_{\theta}^{k+1} \mid s) \\
    &\approx \mathcal{H}(\pi_{\theta}^k \mid s) \\
    &\quad + \langle\nabla \mathcal{H}(\pi_{\theta}^k \mid s), (z^{k+1} - z^k)\rangle.
\end{aligned}
\end{equation}

To derive $\nabla \mathcal{H}(\pi_{\theta}^k \mid s)$, we start from the definition of entropy:

\begin{equation}
\small
\label{eq:entropy_gradient}
\begin{aligned}  
&\nabla_\theta \mathcal{H}(\pi_\theta \mid s) \\
&= \nabla_\theta \left(-\mathbb{E}_{a \sim \pi_\theta(\cdot \mid s)} [\log \pi_\theta(a \mid s)]\right) \\
&= -\nabla_\theta \mathbb{E}_{a \sim \pi_\theta(\cdot \mid s)} [\log \pi_\theta(a \mid s)] \\
&= -\mathbb{E}_{a \sim \pi_\theta(\cdot \mid s)} [\nabla_\theta \log \pi_\theta (a \mid s)] \\
&\quad - \mathbb{E}_{a \sim \pi_\theta(\cdot \mid s)} [\log\pi_\theta(a \mid s) \\
&\quad\quad \times \nabla_\theta\log\pi_\theta(a \mid s)] \\
&= 0 - \mathbb{E}_{a \sim \pi_\theta(\cdot \mid s)} [\log\pi_\theta(a \mid s) \\
&\quad\quad \times \nabla_\theta\log\pi_\theta(a \mid s)] \\
&= -\mathbb{E}_{a \sim \pi_\theta(\cdot \mid s)} [\log\pi_\theta(a \mid s) \\
&\quad\quad \times \nabla_\theta\log\pi_\theta(a \mid s)].
\end{aligned}
\end{equation}

The first term equals zero because $\mathbb{E}_{a \sim \pi_\theta(\cdot \mid s)} [\nabla_\theta \log \pi_\theta (a \mid s)] = \nabla_\theta \sum_a \pi_\theta(a \mid s) = \nabla_\theta 1 = 0$.

\subsection{Softmax Derivative and Centered Logit Updates}

For a softmax policy $\pi_\theta(a \mid s) = \frac{\exp(z_{s,a})}{\sum_{a''} \exp(z_{s,a''})}$, the derivative of the log-probability with respect to any logit parameter $z_{s,a'}$ is:
\begin{equation}
\small
\label{eq:softmax_derivative}
\begin{aligned}
\frac{\partial \log \pi_\theta(a \mid s)}{\partial z_{s, a'}} = \mathbf{1}\{a=a'\} \\
-\pi_\theta(a' \mid s).
\end{aligned}
\end{equation}

We now introduce a centered form of logit updates that exploits the translation invariance of softmax:

\begin{definition}[Centered Logit Change]
\label{def:centered_logit}
Define the \textit{centered logit change} as:
\begin{equation}
\small
\begin{aligned}
\delta z_{s,a} &:= z^{k+1}_{s,a} - z^k_{s,a} \\
&\quad - \mathbb{E}_{a' \sim \pi_\theta^k(\cdot|s)}[z^{k+1}_{s,a'} - z^k_{s,a'}].
\end{aligned}
\end{equation}
This removes the global translation component, which does not affect softmax probabilities.
\end{definition}

Using Equation~\eqref{eq:softmax_derivative}, we can express the first-order entropy change:

\begin{equation}
\small
\label{eq:entropy_change_intermediate}
\begin{aligned}  
&\langle\nabla_\theta \mathcal{H}(\theta^k \mid s), (z^{k+1} - z^k)\rangle \\
&= -\mathbb{E}_{a \sim \pi_\theta^k(\cdot \mid s)}\Big[\log\pi_\theta^k(a \mid s) \\
&\quad \times \sum_{a'} \big(\mathbf{1}\{a=a'\}-\pi_\theta^k(a' \mid s)\big) \\
&\quad \times (z^{k+1}_{s,a'} - z^k_{s,a'})\Big] \\
&= -\mathbb{E}_{a \sim \pi_\theta^k(\cdot \mid s)}\Big[\log\pi_\theta^k(a \mid s) \\
&\quad \times \big(z^{k+1}_{s,a} - z^k_{s,a} \\
&\quad - \sum_{a'} \pi_\theta^k(a' \mid s) (z^{k+1}_{s,a'} - z^k_{s,a'})\big)\Big] \\
&= -\mathbb{E}_{a \sim \pi_\theta^k(\cdot \mid s)}[\log\pi_\theta^k(a \mid s) \cdot \delta z_{s,a}].
\end{aligned}
\end{equation}

Expanding further and using the covariance decomposition $\mathbb{E}[XY] = \mathbb{E}[X]\mathbb{E}[Y] + \text{Cov}(X,Y)$:

\begin{equation}
\small
\label{eq:entropy_change_covariance}
\begin{aligned}  
&\langle\nabla_\theta \mathcal{H}(\theta^k \mid s), (z^{k+1} - z^k)\rangle \\
&= -\mathbb{E}_{a \sim \pi_\theta^k(\cdot \mid s)}[\log\pi_\theta^k(a \mid s) \\
&\quad \times (z^{k+1}_{s,a} - z^k_{s,a})] \\
&\quad + \mathbb{E}_{a \sim \pi_\theta^k(\cdot \mid s)}[\log\pi_\theta^k(a \mid s)] \\
&\quad \times \mathbb{E}_{a' \sim \pi_\theta^k(\cdot \mid s)}[z^{k+1}_{s,a'} - z^k_{s,a'}] \\
&= -\mathbb{E}_{a \sim \pi_\theta^k(\cdot \mid s)}\Big[\big(\log\pi_\theta^k(a \mid s) \\
&\quad - \mathbb{E}_{a \sim \pi_\theta^k(\cdot \mid s)}[\log\pi_\theta^k(a \mid s)]\big) \\
&\quad \times \delta z_{s,a}\Big] \\
&= -\mathbb{E}_{a \sim \pi_\theta^k(\cdot \mid s)}[\Lambda_\theta^k(a|s) \cdot \delta z_{s,a}],
\end{aligned}
\end{equation}
where we used Definition~\ref{def:centered_logprob} in the last step, noting that $-\mathbb{E}[\log\pi_\theta] = \mathcal{H}(\pi_\theta|s)$.

\subsection{GRPO Update Coefficient}

For GRPO, we define the \textit{composite update coefficient}:
\begin{equation}
\small
\label{eq:update_coefficient}
\xi_{i,t}(a) := \mathbb{I}_{\text{clip}} \cdot r_{i,t} \cdot A_{i,t},
\end{equation}
where $\mathbb{I}_{\text{clip}}$ is the clipping indicator function, $r_{i,t} = \frac{\pi_\theta(a_{i,t} \mid s_{i,t})}{\pi_{\text{ref}}(a_{i,t} \mid s_{i,t})}$ is the importance sampling ratio, and $A_{i,t}$ is the advantage estimate. This coefficient combines three essential components: clipping for stability, importance weighting for off-policy correction, and advantage for policy improvement direction.

\begin{lemma}[Centered Logit Update for GRPO]
\label{lemma:grpo_update}
Under the GRPO update rule, the per-sample stochastic logit update satisfies:
\begin{equation}
\small
\begin{aligned}
z_{s,k}^{k+1} - z_{s,k}^k &= \eta\, \xi_{i,t}\,\big(\mathbf{1}\{a=k\} - \pi_\theta(k\mid s)\big).
\end{aligned}
\end{equation}
Specifically, for the sampled action $a$ and other actions $a' \neq a$:
\begin{equation}
\small
\begin{aligned}
z_{s,a}^{k+1} - z_{s,a}^k &= \eta\, \xi_{i,t}\,(1-\pi_\theta(a\mid s)) \\
z_{s,a'}^{k+1} - z_{s,a'}^k &= -\eta\, \xi_{i,t}\,\pi_\theta(a'\mid s).
\end{aligned}
\end{equation}
\end{lemma}

\begin{remark}
By the definition of advantage functions~\cite{sutton2018reinforcement} and GRPO's group normalization~\cite{shao2024deepseekmath}, advantages satisfy $\mathbb{E}_{a\sim\pi_\theta(\cdot\mid s)}[A_{i,t}] = 0$.
\end{remark}

\subsection{Factorized Entropy Change Formula}

Substituting Lemma~\ref{lemma:grpo_update} into Equation~\eqref{eq:entropy_change_covariance}:

\begin{equation}
\small
\label{eq:entropy_change_expansion}
\begin{aligned}  
&\langle\nabla_\theta \mathcal{H}(\theta^k \mid s), (z^{k+1} - z^k)\rangle \\
&= -\sum_{k} \pi_\theta^k(k\mid s)\,\Lambda_\theta^k(k|s) \big(z^{k+1}_{s,k} - z^k_{s,k}\big) \\
&= -\eta\, \xi_{i,t}\,\Big[ \pi_\theta^k(a\mid s)\,\Lambda_\theta^k(a|s) \\
&\quad - \sum_{k} \pi_\theta^k(k\mid s)^2\,\Lambda_\theta^k(k|s) \Big].
\end{aligned}
\end{equation}

Taking expectation over $a\sim\pi_\theta^k(\cdot\mid s)$ under GRPO's group normalization, the entropy change is dominated by:

\begin{equation}
\small
\label{eq:entropy_change_expectation}
\begin{aligned}
\Delta \mathcal{H}(\pi_\theta^k|s) &\approx -\eta\,\mathbb{E}_{a\sim\pi_\theta^k}\Big[ \pi_\theta^k(a\mid s)  \Lambda_\theta^k(a|s)\, \xi_{i,t}(a) \Big].
\end{aligned}
\end{equation}

We now state our main result:

\begin{theorem}[Factorized Entropy Change]
\label{thm:factorized_entropy}
Following GRPO's design~\cite{shao2024deepseekmath} with standard clipping, the first-order entropy change admits:
\begin{equation}
\small
\boxed{
\begin{aligned}
\Delta \mathcal{H}(\pi_\theta^k|s) &\approx -\eta \sum_a \underbrace{\pi_\theta^k(a|s)^2}_{\text{Concentration}} \\
&\quad \times \underbrace{\Lambda_\theta^k(a|s)}_{\text{Info Deviation}} \\
&\quad \times \underbrace{\xi_{i,t}(a)}_{\text{Update Coeff.}}
\end{aligned}
}
\end{equation}
where $\Lambda_\theta^k(a|s) = \log\pi_\theta^k(a|s) + \mathcal{H}(\pi_{\theta}^{k}|s)$ (Definition~\ref{def:centered_logprob}), and:
\begin{itemize}[leftmargin=*,nosep]
    \item $\pi_\theta^k(a|s)^2$: Policy concentration
    \item $\Lambda_\theta^k(a|s)$: Info-theoretic deviation
    \item $\xi_{i,t}(a) = \mathbb{I}_{\text{clip}} \cdot r_{i,t} \cdot A_{i,t}$: Update coefficient (Eq.~\ref{eq:update_coefficient})
\end{itemize}
\end{theorem}

\begin{proof}
From Equation~\eqref{eq:entropy_change_expectation}:
\begin{equation}
\small
\begin{aligned}
&\Delta \mathcal{H}(\pi_\theta^k|s) \\
&\approx -\eta\,\mathbb{E}_{a\sim\pi_\theta^k}\Big[ \pi_\theta^k(a\mid s)  \Lambda_\theta^k(a|s)\, \xi_{i,t}(a) \Big] \\
&= -\eta \sum_a \pi_\theta^k(a|s) \cdot \pi_\theta^k(a\mid s) \times \Lambda_\theta^k(a|s)\, \xi_{i,t}(a) \\
&= -\eta \sum_a \pi_\theta^k(a|s)^2   \Lambda_\theta^k(a|s) \cdot \xi_{i,t}(a),
\end{aligned}
\end{equation}
where the second line expands the expectation as $\mathbb{E}_{a\sim\pi}[f(a)] = \sum_a \pi(a|s) \cdot f(a)$, and the third line simplifies the product. Substituting Definition~\ref{def:centered_logprob} yields the boxed form.
\end{proof}

\begin{remark}
This factorization reveals three orthogonal mechanisms governing entropy dynamics: (i) \textit{concentration} amplifies updates for high-probability actions, (ii) \textit{information deviation} directs change based on distance from maximum entropy, and (iii) \textit{update coefficient} modulates the magnitude via clipped importance-weighted advantages. The multiplicative structure implies that entropy change vanishes when any factor approaches zero, providing natural regularization.
\end{remark}

\section{Gradient-Entropy Relationship: Proof of Expected Gradient Norm Upper Bound}
\label{sec:grad_bound}

This appendix establishes the relationship between the score function's squared L2 norm and policy entropy in softmax policies, then extends it to advantage-weighted updates.

\paragraph{Core Result.}
For softmax policy $\pi_\theta(\cdot|s) = (\pi_1, \ldots, \pi_{|\mathcal{V}|})$ with logits $z(s) = (z_1(s), \ldots, z_{|\mathcal{V}|}(s))$ where $\pi_k = \exp(z_k)/\sum_j \exp(z_j)$, the gradient of $\log \pi_k$ with respect to logits is $\frac{\partial \log \pi_k}{\partial z_i} = \delta_{ik} - \pi_i$ (Kronecker delta $\delta_{ik}$). 

The squared norm for action $k$ is:
\begin{equation}
\small
\begin{aligned}
\|\nabla_z \log \pi_k\|^2 
&= \sum_{i=1}^{|\mathcal{V}|} (\delta_{ik} - \pi_i)^2 \\
&= (1 - \pi_k)^2 + \sum_{i \neq k} \pi_i^2 \\
&= 1 - 2\pi_k + \sum_{j=1}^{|\mathcal{V}|} \pi_j^2.
\end{aligned}
\end{equation}

Taking expectation over actions sampled from $\pi_\theta$:
\begin{equation}
\small
\begin{aligned}
&\mathbb{E}_{a_k \sim \pi_\theta(\cdot|s)}\big[\|\nabla_z \log \pi_k\|^2\big] \\
&= \sum_{k=1}^{|\mathcal{V}|} \pi_k \left( 1 - 2\pi_k + \sum_{j=1}^{|\mathcal{V}|} \pi_j^2 \right) \\
&= \sum_{k=1}^{|\mathcal{V}|} \pi_k - 2\sum_{k=1}^{|\mathcal{V}|} \pi_k^2 \\
&\quad + \left(\sum_{j=1}^{|\mathcal{V}|} \pi_j^2\right) \underbrace{\sum_{k=1}^{|\mathcal{V}|} \pi_k}_{=1} \\
&= 1 - \sum_{k=1}^{|\mathcal{V}|} \pi_k^2.
\end{aligned}
\end{equation}

The term $\sum_{k=1}^{|\mathcal{V}|} \pi_k^2$ is the collision probability (probability that two independent samples from $\pi_\theta$ are identical), ranging from $1/|\mathcal{V}|$ (uniform) to $1$ (deterministic).

\paragraph{Entropy Upper Bound.}
To connect collision probability to Shannon entropy, apply Jensen's inequality. Since $x \mapsto \log x$ is concave, $\sum_k \pi_k \log \pi_k \leq \log(\sum_k \pi_k^2)$. Multiplying by $-1$ and exponentiating:
\begin{equation}
\small
\begin{aligned}
\mathcal{H}(\pi_\theta|s) &= -\sum_{k=1}^{|\mathcal{V}|} \pi_k \log \pi_k \\
&\geq -\log\left(\sum_{k=1}^{|\mathcal{V}|} \pi_k^2\right) \\
&\Rightarrow \sum_{k=1}^{|\mathcal{V}|} \pi_k^2 \geq e^{-\mathcal{H}(\pi_\theta|s)}.
\end{aligned}
\end{equation}

Thus $\mathbb{E}_{a_k \sim \pi_\theta(\cdot|s)}[\|\nabla_z \log \pi_k\|^2] = 1 - \sum_{k=1}^{|\mathcal{V}|} \pi_k^2 \leq 1 - e^{-\mathcal{H}(\pi_\theta|s)}$. High entropy ($\mathcal{H}(\pi_\theta|s) \gg 0$) yields $e^{-\mathcal{H}} \approx 0$ and gradient norm $\approx 1$ (vigorous learning); low entropy ($\mathcal{H}(\pi_\theta|s) \approx 0$) yields $e^{-\mathcal{H}} \approx 1$ and gradient norm $\approx 0$ (stable convergence).

\paragraph{Advantage-Weighted Update Magnitude.}
For policy gradient update with learning rate $\eta$ and advantage $A$, the logit update for sampled action $a_k$ is $\Delta z(s) = \eta A (\mathbf{e}_k - \pi_\theta(\cdot|s))$ where $\mathbf{e}_k$ is the one-hot vector. The squared L2 norm is:
\begin{equation}
\small
\begin{aligned}
\|\Delta z(s)\|_2^2 &= \eta^2 A^2 \Big( 1 - 2\pi_k  + \sum_{j=1}^{|\mathcal{V}|} \pi_j^2 \Big).
\end{aligned}
\end{equation}

Taking expectation over $a_k \sim \pi_\theta$ gives $\mathbb{E}[\|\Delta z(s)\|_2^2] = \eta^2 \mathbb{E}[A^2 (1 - 2\pi_k + \sum_j \pi_j^2)]$. Following GRPO's design~\cite{shao2024deepseekmath} where advantages are sequence-level, the dependence between $A$ and per-token probabilities $\pi_k$ is negligible, allowing the approximation:
\begin{equation}
\small
\begin{aligned}
&\mathbb{E}_{a_k \sim \pi_\theta(\cdot|s)}\big[\|\Delta z(s)\|_2^2\big] \\
&= \eta^2 \mathbb{E}[A^2] \left( 1 - \sum_{k=1}^{|\mathcal{V}|} \pi_k^2 \right) \\
&\leq \eta^2 \mathbb{E}[A^2] \Big( 1 - \exp\big(-\mathcal{H}(\pi_\theta|s)\big) \Big).
\end{aligned}
\end{equation}

This combines logit dynamics with advantage-weighted scaling and entropy bounds. High-confidence tokens ($\pi_k \approx 1$) produce small expected update magnitudes ($\sum_j \pi_j^2 \approx 1$), creating gradient attenuation for confident correct actions—the theoretical basis for gradient compensation in our method.

\section{Gradient Variance Minimization for Dynamic Rollout Allocation}
\label{sec:gradient_variance}

This appendix provides the complete derivation of the optimal rollout allocation formula.

\paragraph{Optimal Allocation via Lagrange Multipliers.}
For prompt $q_i$ with $n_i$ rollouts, the gradient estimator $\hat{g}_i = \frac{1}{n_i}\sum_{k=1}^{n_i} g_{i,k}$ (where $g_{i,k} = \frac{1}{G}\sum_{j=1}^G \hat{A}_{j,k} \nabla_\theta \log \pi_\theta(o_{j,k}|q_i)$ with $\hat{A}_{j,k}$ the group-normalized advantage) has variance $\text{Var}[\hat{g}_i] = \sigma_i^2 / n_i$ where $\sigma_i^2 = \mathbb{E}[\|g_{i,k}\|_2^2]$ is the single-sample gradient variance. 

Minimizing total variance $\sum_{i=1}^N \sigma_i^2/n_i$ subject to budget constraint $\sum_i n_i = B$, we construct Lagrangian $\mathcal{L} = \sum_i \sigma_i^2/n_i + \lambda(\sum_i n_i - B)$. The first-order condition $\partial \mathcal{L}/\partial n_i = -\sigma_i^2/n_i^2 + \lambda = 0$ gives $n_i = \sigma_i/\sqrt{\lambda}$. Substituting into the constraint $\sum_i \sigma_i/\sqrt{\lambda} = B$ yields $\sqrt{\lambda} = (\sum_k \sigma_k)/B$, thus:
\begin{equation}
\small
n_i^* = B \cdot \frac{\sigma_i}{\sum_{k=1}^N \sigma_k}.
\end{equation}

\paragraph{Variance Decomposition.}
To implement this principle, we decompose $\sigma_i^2 = \mathbb{E}[\|g_{i,k}\|_2^2]$ where $g_{i,k} = \frac{1}{G}\sum_{j=1}^G A_{j,k} \nabla_\theta \log \pi_\theta(o_{j,k}|q_i)$. The squared norm is:
\begin{equation}
\small
\begin{aligned}
\|g_{i,k}\|_2^2 &= \frac{1}{G^2}\Big\|\sum_{j=1}^G A_{j,k} \times \nabla_\theta \log \pi_\theta(o_{j,k}|q_i)\Big\|_2^2.
\end{aligned}
\end{equation}

Following the treatment in high-dimensional gradient estimation~\cite{schulman2017proximal, yao2025optimizing}, $\mathbb{E}[\langle \nabla_j, \nabla_{j'} \rangle] \approx 0$, yielding:
\begin{equation}
\small
\begin{aligned}
\mathbb{E}[\|g_{i,k}\|_2^2] &\approx \frac{1}{G}\mathbb{E}\Big[A^2  \|\nabla_\theta \log \pi_\theta(o|q_i)\|_2^2\Big].
\end{aligned}
\end{equation}

For GRPO with group-normalized advantages~\cite{shao2024deepseekmath}, the gradient variance exhibits positive correlation with both reward variance and expected gradient magnitude:
\begin{equation}
\small
\begin{aligned}
\mathbb{E}[\|g_{i,k}\|_2^2] &\propto \text{Var}(R)   \mathbb{E}[\|\nabla_\theta \log \pi_\theta(o|q_i)\|_2^2].
\end{aligned}
\end{equation}

For response $o = \{o_1, \ldots, o_T\}$, the policy gradient decomposes as $\nabla_\theta \log \pi_\theta(o|q_i) = \sum_{t=1}^T \nabla_\theta \log \pi_\theta(o_t|q_i, o_{<t})$. Applying the same approximation to token-level gradients:
\begin{equation}
\small
\begin{aligned}
&\mathbb{E}[\|\nabla_\theta \log \pi_\theta(o|q_i)\|_2^2] \approx \mathbb{E}\Big[\sum_{t=1}^T \|\nabla_\theta \log \pi_\theta(o_t|q_i, o_{<t})\|_2^2\Big].
\end{aligned}
\end{equation}

From Appendix~\ref{sec:grad_bound}, the expected squared gradient norm for token $t$ is $\mathbb{E}[\|\nabla_\theta \log \pi_\theta(o_t|q_i, o_{<t})\|_2^2] = 1 - C(P_t)$ where $C(P_t) = \sum_{k=1}^{|\mathcal{V}|} \pi_k^2(q_i, o_{<t})$ is the collision probability. Defining average collision $\bar{C}_i = \mathbb{E}_{o \sim \pi_\theta(\cdot|q_i)}[\frac{1}{|o|}\sum_{t=1}^{|o|} C(P_t)]$:

\begin{equation}
\small
\sigma_i^2 \propto \text{Var}(R) \cdot (1 - \bar{C}_i).
\end{equation}

Computing $C(P_t) = \sum_{k=1}^{|\mathcal{V}|} \pi_k^2$ for all tokens is expensive. Since high entropy corresponds to low collision probability (uniform distributions have low $C(P)$, concentrated distributions have high $C(P)$), we use entropy as a proxy. Defining average per-token entropy $\bar{\mathcal{H}}_i = \mathbb{E}_{o \sim \pi_\theta(\cdot|q_i)}\left[\frac{1}{|o|}\sum_{t=1}^{|o|} \mathcal{H}(\pi_\theta|q_i, o_{<t})\right]$, we approximate:
\begin{equation}
\small
\sigma_i \propto \text{std}(R) \cdot f(\bar{\mathcal{H}}_i)
\end{equation}
where $f(\cdot)$ is an increasing function. Using $f(\mathcal{H}) = \sqrt{1 - e^{-\mathcal{H}}}$ captures the relationship: higher entropy leads to larger gradient variance.

\paragraph{Connection to Bernoulli Variance.}
For binary rewards, $\text{Var}(R) = p(1-p)$ where $p$ is success probability. The unbiased estimator using historical statistics ($k_i$ correct responses out of $G_i$ total rollouts) is:
\begin{equation}
\small
\begin{aligned}
P_i = \frac{k_i(G_i - k_i)}{G_i(G_i - 1)}, \quad \mathbb{E}[P_i] = p_i(1-p_i).
\end{aligned}
\end{equation}

This $P_i$ serves as a proxy for $\sigma_i$ because: (i) it directly estimates $\text{Var}(R)$ through Bernoulli variance, (ii) empirical accuracy $\hat{p}_i = k_i/G_i$ near the variance-maximizing value correlates with high policy entropy (models actively exploring multiple solution paths), and (iii) together, $P_i$ identifies high-$\sigma_i$ problems per the variance decomposition. The water-level allocation algorithm implements $n_i^* \propto P_i$ with constraints $G_{\min}, G_{\max}$.

\paragraph{Variance Reduction Guarantee.}
The optimal allocation achieves total variance $\text{Var}^* = \frac{1}{B}(\sum_i \sigma_i)^2$, while uniform allocation ($n_i = B/N$) yields $\text{Var}_{\text{uniform}} = \frac{N}{B}\sum_i \sigma_i^2$. Their ratio is:
\begin{equation}
\small
\frac{\text{Var}^*}{\text{Var}_{\text{uniform}}} = \frac{(\sum_i \sigma_i)^2}{N \sum_i \sigma_i^2} \leq 1,
\end{equation}
where the inequality follows from Cauchy-Schwarz, with equality only when all $\sigma_i$ are equal. When gradient variances are heterogeneous across problems (typical in RLVR), this adaptive allocation provides substantial variance reduction.

\section{Dynamic Rollout Allocation: Operational Details.}
\label{sec:DRA_detail}

Our allocation mechanism determines rollout quantities for each prompt in the training batch, operating independently of prompt selection strategies (e.g., random sampling, curriculum learning). Given a batch of $N$ prompts and total rollout budget $B$, the water-level algorithm distributes $n_i$ rollouts per prompt proportionally to variance proxies $P_i$ while respecting minimum and maximum constraints.

The allocation process guarantees complete budget utilization through iterative refinement. Initially, each prompt receives a baseline allocation based on its $P_i$ value. When constraints become active (some prompts hit minimum or maximum bounds), residual budget is redistributed among unconstrained prompts using the same proportional rule. This continues until all budget is assigned and $\sum_i n_i = B$ holds exactly. If the total minimum requirement exceeds available budget, we proportionally scale down the minimum allocation to ensure feasibility.

For prompts with insufficient historical data (few prior observations), computing reliable variance proxies $P_i$ is infeasible. In such cases, we apply uniform allocation as a fallback strategy, distributing rollouts equally among data-scarce prompts. As prompts accumulate responses through repeated selection in training batches, their historical statistics become sufficient to support adaptive allocation. This transition from uniform to variance-driven distribution occurs naturally as the $P_i$ estimates stabilize with increased sample counts.

The proportional allocation naturally adapts to estimation errors. If historical $P_i$ overestimates current variance (due to policy improvement), subsequent rollouts reveal lower empirical variance, which updates $P_i$ for future allocations. Conversely, underestimated prompts exhibit higher gradient variance, triggering increased allocation in later iterations. This feedback loop ensures the distribution tracks evolving problem difficulty without manual tuning. In boundary cases where all $P_i$ values become negligibly small (indicating near-deterministic outcomes), the algorithm defaults to uniform distribution to maintain numerical stability.

\begin{figure*}[!ht]
\subfigure[Qwen2.5-Math-1.5B]{
\begin{minipage}[t]{0.48\linewidth}
\centerline{\includegraphics[width=1\hsize]{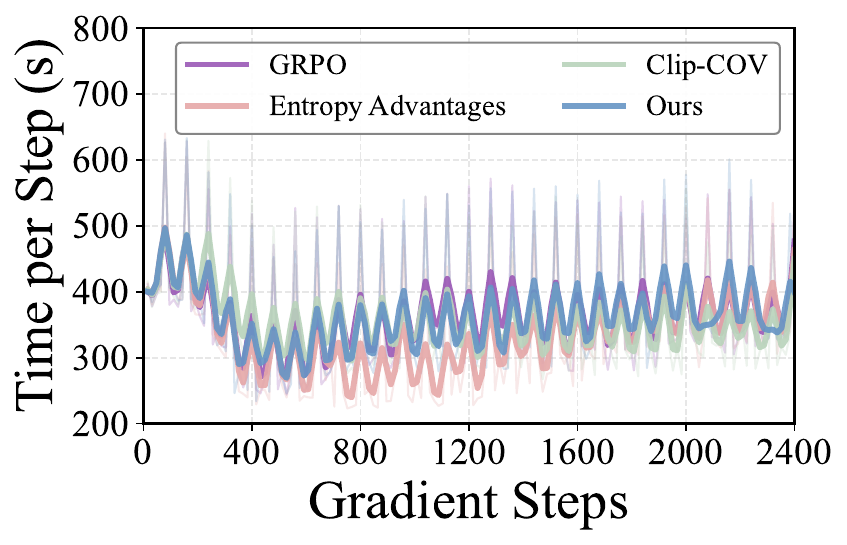}}
\label{fig:timing_1.5b}
\end{minipage}}
\subfigure[Qwen2.5-Math-7B]{
\begin{minipage}[t]{0.48\linewidth}
\centerline{\includegraphics[width=1\hsize]{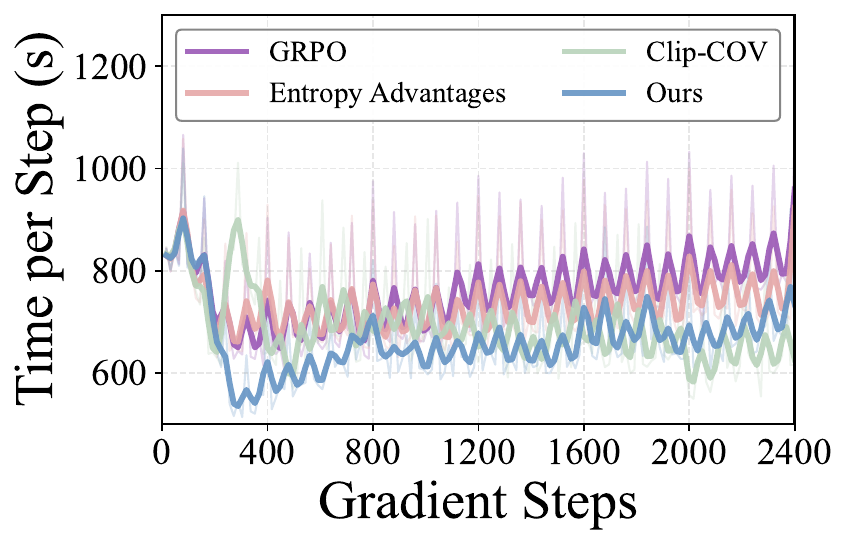}}
\label{fig:timing_7b}
\end{minipage}}
\caption{Per-step training time on Qwen2.5-Math models. Bold lines: smoothed; light lines: raw measurements.}
\label{fig:timing}
\end{figure*}

\section{Detailed Training Configuration}
\label{sec:training_config}

All experiments are conducted on 8× NVIDIA A100 80GB GPUs per node (1-2 nodes depending on model scale) with mixed precision training (bfloat16) and FSDP for distributed training. The actor model uses Tensor Parallelism (TP=2) during rollout generation with vLLM for efficient inference. Training prompts are left-truncated to 2048 tokens if exceeding the maximum length, and we apply Math-Verify~\cite{mathverify} for answer verification during training and evaluation. The training dataset is shuffled at the beginning of each epoch, and rollout budget allocation is updated dynamically based on historical success statistics accumulated across iterations. Table~\ref{tab:training_config} presents the complete hyperparameters and configuration details for our experiments.

\begin{table*}[h]
\centering

\setlength{\tabcolsep}{10pt}
\renewcommand{\arraystretch}{1.3}
\caption{Detailed training configuration for DynaMO experiments.}
\label{tab:training_config}
\begin{tabular}{lll}
\toprule
\textbf{Category} & \textbf{Parameter} & \textbf{Value} \\
\midrule
\multicolumn{3}{c}{\textit{\textbf{Training Hyperparameters}}} \\
\midrule
Batch Configuration & Generation Batch Size & 512 \\
 & Update Batch Size & 32 \\
 & PPO Mini-batch Size & 32 \\
Rollout Allocation & Average Rollouts per Prompt & 16 \\
 & Dynamic Range & [8, 24] \\
Optimization & Learning Rate (Actor) & $1 \times 10^{-6}$ \\
 & Learning Rate (Critic) & $1 \times 10^{-5}$ \\
 & Weight Decay & 0.1 / 0.01 \\
 & Gradient Clipping & 1.0 \\
 & Warmup Steps & 10 \\
 & Warmup Style & constant \\
\midrule
\multicolumn{3}{c}{\textit{\textbf{GRPO-Specific Settings}}} \\
\midrule
Clipping & Clip Ratio & 0.2 \\
Entropy & Entropy Coefficient & 0 \\
KL Penalty & KL Coefficient & 0.0 \\
Advantage Estimator & Type & GRPO  \\
 & Normalize by Std & True \\
 & Gamma (Discount Factor) & 1.0 \\
 & Lambda (GAE) & 1.0 \\
\midrule
\multicolumn{3}{c}{\textit{\textbf{Inference \& Rollout}}} \\
\midrule
Sampling Strategy & Temperature & 1.0 \\
 & Top-p & 1.0 \\
 & Top-k & -1 (disabled) \\
Sequence Lengths & Max Prompt Length & 2048 \\
 & Max Response Length & 8192 \\
Rollout Engine & Backend & vLLM \\
 & Tensor Parallel Size & 2 \\
 & GPU Memory Utilization & 0.8 \\
 & Max Num Sequences & 1024 \\
 & Max Batched Tokens & 10240 \\
\bottomrule
\end{tabular}
\end{table*}

\section{Detailed Description of Benchmarks}
\label{sec:benchmark_details}

To comprehensively evaluate mathematical reasoning capabilities, we employ six widely-adopted benchmarks spanning competition-level challenges (AIME, AMC), curriculum-aligned assessments (MATH), and specialized STEM domains (Minerva, OlympiadBench). These benchmarks encompass diverse mathematical subfields and problem formats, enabling rigorous evaluation across multiple dimensions of reasoning proficiency. Table~\ref{tab:benchmark_details} summarizes the key characteristics of each benchmark.

\begin{table*}[htbp]
\centering
\renewcommand{\arraystretch}{1.8}
\resizebox{\textwidth}{!}{
\begin{tabular}{@{} l p{5cm} p{7.5cm} @{}}
\toprule
\textbf{Benchmark} & \textbf{Core Description} & \textbf{Key Characteristics} \\
\midrule

\textbf{AIME} & 
American Invitational Mathematics Examination - high school competition & 
\begin{itemize}[leftmargin=*,nosep,itemsep=2pt]
\item 15 challenging problems per round
\item Integer answers (0-999)
\item Algebra, Geometry, Number Theory, Combinatorics
\item Multi-step reasoning required
\item Top AMC performers participate
\end{itemize} \\
\addlinespace[6pt]

\midrule

\textbf{AMC} & 
American Mathematics Competitions - tiered assessment system & 
\begin{itemize}[leftmargin=*,nosep,itemsep=2pt]
\item Three tiers (AMC 8/10/12) for different grades
\item 25 multiple-choice problems per tier
\item Curriculum-aligned design
\item Comprehensive secondary mathematics coverage
\item Standardized difficulty progression
\end{itemize} \\
\addlinespace[6pt]

\midrule

\textbf{MATH-500} & 
Curated subset from MATH dataset & 
\begin{itemize}[leftmargin=*,nosep,itemsep=2pt]
\item 500 problems from comprehensive collection
\item Seven domains (Algebra, Geometry, Number Theory, etc.)
\item Five difficulty levels
\item Formal mathematical reasoning
\item Step-by-step solution verification
\end{itemize} \\
\addlinespace[6pt]

\midrule

\textbf{Minerva Math} & 
Technical STEM benchmark & 
\begin{itemize}[leftmargin=*,nosep,itemsep=2pt]
\item Undergraduate to graduate difficulty
\item Physics, Chemistry, Biology applications
\item Symbolic manipulation and formula derivation
\item Domain-specific knowledge integration
\item Quantitative problem-solving
\end{itemize} \\
\addlinespace[6pt]

\midrule

\textbf{OlympiadBench} & 
Bilingual Olympiad-level benchmark & 
\begin{itemize}[leftmargin=*,nosep,itemsep=2pt]
\item Large-scale competition problem collection
\item Bilingual (English and Chinese)
\item Mathematics, Physics, Chemistry, Biology
\item Theorem-proving and open-ended problems
\item Multimodal inputs (text, diagrams, equations)
\end{itemize} \\

\bottomrule
\end{tabular}
}
\caption{Characteristics of mathematical reasoning benchmarks used in our evaluation. These benchmarks provide comprehensive coverage across difficulty levels (secondary to graduate), problem formats (multiple-choice, integer answers, open-ended), and mathematical domains (pure and applied mathematics), enabling thorough assessment of reasoning capabilities.}
\label{tab:benchmark_details}
\end{table*}

\section{Efficiency Analysis}

Figure~\ref{fig:timing} presents per-step training time across model scales. DynaMO maintains comparable efficiency to baselines on both 1.5B and 7B models, as the additional computations—Bernoulli variance estimation, water-level allocation, and token-level modulation—involve only lightweight arithmetic operations on existing logits and statistics, incurring negligible overhead relative to model inference.

\end{document}